\documentclass[journal]{IEEEtran}
\usepackage{array}
\usepackage{textcomp}
\usepackage{stfloats}
\usepackage{url}
\usepackage{verbatim}
\usepackage{graphicx}

\usepackage{cite}
\usepackage{amsmath,amssymb,amsfonts}
\usepackage{algorithm,algorithmic}
\usepackage{textcomp}
\usepackage{xcolor}
\usepackage{siunitx}
\usepackage{upgreek}
\usepackage{fancyhdr}
\usepackage{makecell}
\usepackage{marvosym}
\usepackage{amsthm} 
\usepackage{booktabs}
\usepackage{xspace}
\usepackage{anyfontsize}

\newcommand{\method}{MAGPIE\xspace}
\newcommand{\methodfull}{\textbf{M}ulti-\textbf{AG}ent \textbf{P}reference-\textbf{I}ntegrated l\textbf{E}arning (\method)\xspace}

\newtheorem{assumption}{\bf Assumption}
\newtheorem{definition}{\bf Definition}
\newtheorem{theorem}{\bf Theorem}
\newtheorem{lemma}{\bf Lemma}

\newtheorem{corollary}{\bf Corollary}
\newtheorem*{remark}{Remark}

\makeatletter
\renewcommand{\eqref}[1]{\textup{Eq.~\ref{#1}}}
\makeatother

\newcommand{\revise}[1]{#1}
\usepackage[colorlinks,
            linkcolor=blue,
            anchorcolor=black,
            urlcolor=magenta,
            citecolor=blue]{hyperref}

\begin{document}

\title{Multi-Agent Reinforcement Learning via Agent-Specific Preference}

\author{Ni Mu$^*$, Yao Luan$^*$, Yiqin Yang$^\dagger$, Qing-Shan Jia$^\dagger$
        % <-this % stops a space
\thanks{Ni Mu, Yao Luan and Qing-Shan Jia are with CFINS, Department of Automation, 
BNRist, 
Beijing Key Laboratory of Embodied Intelligence Systems, 
Institute for Embodied Intelligence and Robotics, 
Tsinghua University, Beijing 100084, China, 
{\tt\small \{mn23@mails., luany23@mails., jiaqs@\} tsinghua.edu.cn}. 
Yiqin Yang is with the Institute of Automation, Chinese Academy of Sciences, Beijing 100190, China, {\tt\small yiqin.yang@ia.ac.cn}. 
$^*$Equal contribution. 
$^\dagger$Corresponding author. 
}%
}

% The paper headers
\markboth{Journal of \LaTeX\ Class Files,~Vol.~14, No.~8, August~2021}%
{Shell \MakeLowercase{\textit{et al.}}: A Sample Article Using IEEEtran.cls for IEEE Journals}

% \IEEEpubid{0000--0000/00\$00.00~\copyright~2021 IEEE}
% Remember, if you use this you must call \IEEEpubidadjcol in the second
% column for its text to clear the IEEEpubid mark.

\maketitle

\begin{abstract}
Multi-agent reinforcement learning (MARL) is a powerful framework for solving complex collaborative tasks, but it relies heavily on well-defined global reward functions. 
Designing such rewards is challenging, especially in systems with heterogeneous agents, where a single scalar objective may fail to capture diverse behaviors. 
In this paper, we introduce \methodfull, which addresses these challenges through agent-specific preference modeling. 
Each agent is evaluated by a dedicated expert through preference signals, eliminating the need for global evaluation. 
We theoretically prove that optimizing these decentralized preferences converges to a Nash equilibrium policy. 
To integrate local preferences into a coherent global objective, we construct agent-specific reward models from preference data and combine them via a monotonic aggregation mechanism. 
We further prove that optimizing this aggregate reward model is equivalent to training the Nash equilibrium policy. 
Extensive experiments on benchmark multi-agent tasks and a sequential production line task show that \method achieves performance comparable to reward-engineered baselines, demonstrating its potential to facilitate policy learning in scenarios where precise reward engineering is impractical. 
\end{abstract}

\def\abstractname{Note to Practitioners}
\begin{abstract}
Multi-agent systems are widely used in modern engineering applications. 
For example, autonomous vehicle fleets coordinate to prevent collisions while maintaining efficiency, and industrial manufacturing lines work together to meet production targets without causing buffer overflows. 
Multi-agent reinforcement learning (MARL) provides a powerful framework for enabling such collaboration, but its success depends heavily on well-designed reward functions.
Designing these rewards is often challenging, especially when agents play distinct roles, as it is difficult to translate complex interactions and diverse agent behaviors into precise numerical signals. 
In contrast, providing comparative feedback on preferred behaviors is often more intuitive than specifying explicit mathematical rewards. 
In this paper, we introduce \methodfull, a framework that leverages agent-specific preference signals in the multi-agent learning process. 
\method learns agent-specific reward models and combines them into a unified global objective using monotonic aggregation. 
By optimizing this objective, we can derive Nash equilibrium solutions. 
\revise{Importantly, preferences can be provided by lightweight automated rules or domain-specific heuristics, eliminating the need for costly human annotators. }
\method is effective, easy to implement, and particularly suitable for complex systems where traditional reward design is impractical. 
\end{abstract}

\begin{IEEEkeywords}
Reinforcement learning, Multi-agent system, Preference-based optimization, Nash equilibrium. 
\end{IEEEkeywords}

\section{Introduction} \label{sec:intro}

Multi-agent systems are increasingly deployed in real-world scenarios, such as autonomous vehicle fleets navigating urban areas \cite{yu2020distributed}, dynamic user scheduling in wireless networks \cite{wang2025populations}, and teams of agents in strategic games \cite{berner2019dota}.
These systems require multiple heterogeneous agents to collaborate effectively while adapting to dynamic environments. 
Multi-agent reinforcement learning (MARL) has proven effective in these coordination scenarios \cite{mappo}. 
However, its success typically relies on a well-designed global reward function, which quantifies the benefits of joint actions \cite{vdn, qmix}. 
In complex scenarios, where agents are heterogeneous, i.e., agents have diverse roles and objectives \cite{wang2025beyond}, designing such a single reward function can be both difficult and infeasible \cite{ibarz2018reward}. 
This challenge motivates the exploration of alternative methods that are more intuitive and human-centric to guide policy learning.

Preference-based reinforcement learning (PbRL) \cite{pebble, mu2024sepoa, mu2025clarify, luan2025stair, mu2025magpie_cac} emerges as a promising solution, which utilizes human comparisons of trajectory segments to guide policy learning, bypassing the challenges of complex reward engineering. 
PbRL has shown effectiveness in single-agent domains, such as robotics \cite{ni2025senior}, large language model alignment \cite{2022_RLHF_LLM}, and energy optimization \cite{mu2025pbmorl}, all without explicit rewards \cite{luan2026collie}. 
However, directly extending PbRL to multi-agent settings introduces unique challenges. 
First, humans may struggle to evaluate collective team behavior, especially as team size and agent heterogeneity increase. 
Existing methods \cite{kou2025offline, bui2025omapl, zhu2024decoding} often require experts to assess entire team trajectories, which involve diverse agents. This can produce noisy and inconsistent preference signals \cite{mu2024sepoa}. 
Also, current approaches \cite{kang2025dpm, zhang2025multiagent} struggle to accommodate heterogeneous agents, as they assess agents' policies by comparing behaviors across different agents, which implicitly assumes agent homogeneity and comparability. 
This assumption fails in systems where agents have distinct roles, capabilities, and objectives. 
To the best of our knowledge, no existing method addresses these challenges while eliminating the need for reward engineering. 
This demonstrates the need for a novel approach to preference-driven multi-agent coordination.

In this paper, we introduce \methodfull, a framework that addresses the above gap. 
As illustrated in Figure \ref{fig:method_architecture}, \method decomposes the complex global evaluation into a set of feasible, specialized local evaluations, and then integrates them with theoretical guarantees.
Specifically, each agent is assigned a dedicated expert, who provides \emph{agent-specific preferences} by evaluating the behavior of their assigned agent. 
Each dedicated expert only focuses on the agent's individual contribution and collaborative effectiveness, which reduces cognitive load and naturally supports heterogeneous agents. 
We theoretically show that, conceptually, using agent-specific preferences to iteratively improve each agent's policy leads to a Nash equilibrium policy (Theorem \ref{thm:1_nash_iteration}). 
To connect local preferences with practical policy optimization, we train independent reward models for each agent using preference data, and then propose a \emph{monotonic aggregation} mechanism to aggregate these local rewards into a coherent global reward. 
We prove that maximizing this aggregated reward guarantees a joint Nash equilibrium policy (Theorem \ref{thm:2_sum_nash}, \ref{thm:4_epsilon_nash}). 
These theories allow standard MARL algorithms to integrate seamlessly with our framework by replacing their reward signal with the learned aggregated reward.

To demonstrate the effectiveness of \method, we conduct experiments on several benchmark MARL tasks with heterogeneous agent roles and coordination challenges. 
Results show that \method achieves performance comparable to the ``oracle'' reward engineering baseline, which uses ground-truth rewards to learn the policy. 
Additionally, we evaluate \method on a custom sequential mobile phone production line task. 
In this task, \method similarly matches the oracle's performance, demonstrating its potential for practical deployment in complex, real-world scenarios. 
Through this work, we aim to provide a preference-driven optimization framework for heterogeneous multi-agent collaboration, facilitating the broader application of MARL in complex tasks, where reward engineering is prohibitive.

\begin{figure*}[t!]
\centering
\includegraphics[width=0.9\linewidth]{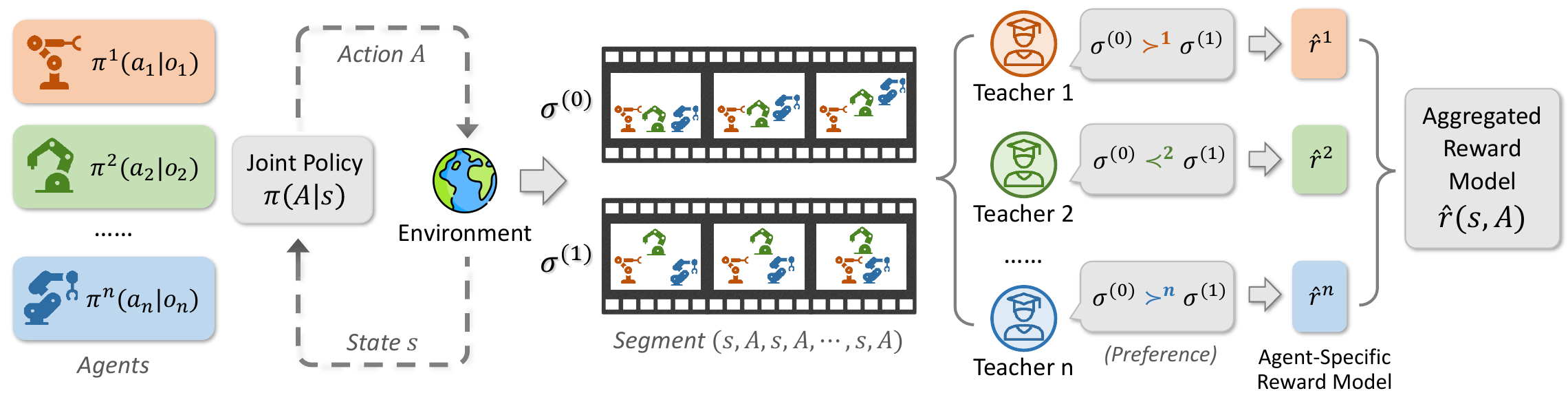}
\caption{
An overview of \method. The agent interacts with the environment and learns a multi-agent policy, without a predefined reward function. 
Instead, $n$ external experts act as teachers provide agent-specific preferences between trajectory segment pairs, which are used to train decentralized reward models $\hat{r}^i$. 
These local models are aggregated via a monotonic function into a global reward $\hat{r}$, enabling standard MARL algorithms to optimize the policy. 
}
\label{fig:method_architecture}
\end{figure*}

The main contributions of this paper are as follows:
\begin{itemize}
\item We establish a theoretical framework for preference-based MARL with agent-specific preference evaluation, proving that the $n$-to-$n$ expert-agent preference mechanism can guide the learning toward Nash equilibrium multi-agent policies (Theorems \ref{thm:1_nash_iteration}, \ref{thm:2_sum_nash}, \ref{thm:4_epsilon_nash}). 

\item We propose \method, which implements this framework. 
By constructing Bradley-Terry models and optimizing the cross-entropy loss function, \method learns local reward models aligned with agent-specific preferences (Theorem \ref{thm:reward_consistency}). 
Using additive monotonic aggregation and a regularization mechanism, \method integrates with standard MARL algorithms to optimize multi-agent policies. 

\item We conduct extensive experiments to evaluate \method on multiple benchmark tasks and a custom smart manufacturing task. 
Results show that \method achieves performance comparable to oracle methods using only decentralized preference signals, validating its potential for real-world applications. 
\end{itemize}

% paper structure
The remaining sections are organized as follows. 
Section \ref{sec:related_work} reviews the related work. 
Section \ref{sec:pblm_fmlation} introduces preliminaries and the problem formulation.
Section \ref{sec:method} presents the theoretical guarantees and the \method algorithm. 
Section \ref{sec:experiment} describes the experimental setting and discusses the results. 
Section \ref{sec:conclusion} concludes the paper.

\section{Related Work}
\label{sec:related_work}

\subsection{Multi-Agent Reinforcement Learning}

Reinforcement learning (RL) has achieved remarkable success in single-agent decision-making tasks, such as game-playing \cite{silver2017mastering, berner2019dota}, robotic control \cite{liu2021deep}, and data center cooling management \cite{mu2023integrating, mu2024large_scale_DC}. 
Building on this success, RL has been extended to multi-agent systems, leading to the development of multi-agent reinforcement learning (MARL) \cite{chang2022emapp, wang2025beyond, wang2025solo, yang2026globediff}. 
As a powerful framework, MARL enables coordination among agents and solves complex tasks, including autonomous vehicle coordination \cite{yu2020distributed}, distributed resource allocation \cite{hady2025multi}, and multi-robot collaboration \cite{oroojlooy2023review}.

The centralized training with decentralized execution (CTDE) framework is commonly used in MARL, which enables distributed execution while ensuring effective global performance through centralized training. 
Notable methods include value decomposition networks (VDN) \cite{vdn} and QMIX \cite{qmix}, which learn decomposed agent-specific value functions, and then combine them monotonically into a global value function. 
More recently, multi-agent proximal policy optimization (MAPPO) \cite{mappo} has emerged as a strong on-policy policy gradient method.

Despite these advancements, the effectiveness of cooperative MARL algorithms heavily relies on a well-defined global reward function $r(s, A)$, which accurately captures the immediate contribution of joint actions $A$ in state $s$. 
Designing such a function is often challenging or impractical, causing a ``reward engineering`` challenge.
The challenges can be classified into three key aspects. 
First, a sparse reward function provides limited feedback in complex tasks, where agents only receive rewards after achieving milestones, such as task completion. 
This lack of intermediate guidance reduces sample efficiency and hinders exploration \cite{rengarajan2022reinforcement}. 
Second, designing reward functions to provide dense signals often requires extensive domain knowledge \cite{eschmann2021reward}, especially in heterogeneous systems, where agents with distinct roles and capabilities must coordinate \cite{du2019liir}. 
For example, in cooperative exploration tasks, balancing rewards for mobile robots that explore terrain, and stationary robots that guard key locations, can lead to fragile or suboptimal designs. 
Third, many desirable real-world behaviors, such as ``fluid teamwork'' or ``graceful motion'', are difficult to define mathematically \cite{Rekik2025quality}.
This makes it challenging to align human intent with numerical rewards through manual reward engineering.
In summary, these challenges illustrate the limitations of traditional MARL in real-world settings, motivating more natural, human-centric supervision signals.

\subsection{Preference-based Guidance for MARL}

Preference-based RL (PbRL) has proven effective and is now widely applied across various domains, as it avoids complex reward engineering through the use of pairwise comparisons, i.e., ``trajectory A is preferred over trajectory B'', rather than absolute reward values \cite{pbrl_basic2017, pebble, mu2024pbmorl_cac, mu2025clarify}. 
Its applications range from robotic control \cite{ni2025senior} to fine-tuning large language models (LLMs) \cite{2022_RLHF_LLM}. 

Building on this success, recent studies have explored how to apply preference-based supervisory signals to guide policy learning in multi-agent systems \cite{kou2025offline, bui2025omapl, zhu2024decoding, kang2025dpm, zhang2025multiagent, kim2025human}. 
However, multi-agent settings introduce unique challenges that naive extensions fail to address, resulting in three key limitations in the current literature. 

The first limitation arises from global preference evaluation. 
As team size and agent diversity increase, evaluating entire multi-agent trajectories becomes impractical. 
Prior works \cite{kou2025offline, bui2025omapl, zhu2024decoding} directly extend single-agent PbRL by evaluating multi-agent trajectories as a whole. 
While conceptually valid, this is cognitively overwhelming and error-prone in complex scenarios. 
For example, in MOBA games, human evaluators must assess attackers, defenders, and supporters collectively, despite their differing roles and contributions. 
This often leads to cognitive overload and evaluation errors \cite{mu2024sepoa, berner2019dota}.

The second limitation is the assumption of agent homogeneity. 
To address the coarseness of global evaluations as mentioned above, some methods \cite{kang2025dpm, kim2025human} compare individual agents' behaviors within the same trajectory, i.e., ``Agent i contributed more than Agent j''. 
However, this implicitly assumes that agents have similar roles and behaviors, which is incompatible with heterogeneous multi-agent systems, where agents differ in observation spaces, action spaces, and role-specific objectives. 
In such cases, direct comparisons are either meaningless or misleading. 

The third limitation is partial reliance on ground-truth rewards. 
Some methods \cite{kim2025human} incorporate ground-truth rewards as auxiliary supervision during training,
This contradicts the core goal of PbRL, which is to eliminate reward engineering. 
As a result, these methods become inapplicable when ground-truth rewards are unavailable.

In summary, to our best knowledge, current research lacks a framework that simultaneously avoids complex global evaluations, supports heterogeneous agents, and eliminates reward engineering. 
To address these gaps, we propose a framework, \method, which uses local preferences to evaluate individual agent contributions, coupled with a provably monotonic aggregation mechanism to combine these local assessments into a coherent global objective. 
\method overcomes all three limitations: it avoids the complexity of global evaluations, supports agent heterogeneity through localized assessments, and relies solely on preference data without requiring ground-truth rewards.

\section{Problem Formulation} \label{sec:pblm_fmlation}

In this section, we first present the multi-agent MDP settings and the centralized training with decentralized execution (CTDE) framework, which is commonly used in MARL \cite{vdn}. 
Then, we provide the preference formulation in our settings.

\subsection{The Multi-Agent MDP Definition and CTDE Framework}
\label{subsec:pre_marl}

% single-agent MDP
In single-agent settings, a Markov Decision Process (MDP) is represented as a tuple $(\mathcal{S}, \mathcal{A}, P, r, \gamma)$, where $\mathcal{S}$ and $\mathcal{A}$ denote the state and action spaces, respectively. 
The state transition probability is defined as $P(s'|s, a): \mathcal{S} \times \mathcal{A} \times \mathcal{S} \rightarrow [0, 1]$, while the immediate reward is $r(s, a): \mathcal{S} \times \mathcal{A} \rightarrow \mathbb{R}$. The discount factor $\gamma \in (0, 1)$ balances short-term and long-term rewards.
The target is to learn a policy $\pi$ that maximizes the expected total discounted reward $J(\pi)=\mathbb{E}_{\pi,s_0}[\sum_{t=0}^\infty \gamma^t r(s_t,a_t)]$.

% Dec-POMDP
Decentralized partially observable Markov decision processes (Dec-POMDP) extend the MDP framework to cooperative multi-agent systems, where multiple agents work together to achieve a common goal. 
Formally, A Dec-POMDP with $n$ agents is defined as a tuple $(\mathcal{S}, \mathcal{A}, \mathcal{O}, r, P, n, \gamma)$, where $\mathcal{S}$ represents the global state space, and $\mathcal{A}$ is the global action space, with each agent $i$ selecting an action $a^i$, composing the joint global action $A =(a^1, \dots, a^n)\in\mathcal{A}$. 
Agents have limited access to the global state $s$, while making local observations $o^i = O(s, i)$ from the global state, drawn from the observation space $\mathcal{O}$. 
The shared reward function $r(s, A)$ depends on the global state $s$ and the joint action $A$. 
% 
% multi-agent objective
The policy for each agent $i$ is represented as $\pi^i(a^i | o^i)$, conditioned on its local observation. 
The agents' objective in a Dec-POMDP is to maximize the expected discounted return:
\begin{equation}
    \max J(\theta) = \mathbb{E}_{(A_t, s_t)} \left[ \sum_t \gamma^t r(s_t, A_t) \right]
\end{equation}
To achieve this, agents must collaboratively select their actions to optimize the joint reward.

% CTDE
Centralized training for decentralized execution (CTDE) is a widely adopted framework for solving Dec-POMDPs. 
In CTDE, agents are trained with global information, such as the global state $s$ and the joint action $A$. 
During execution, however, each agent acts independently based on its local observation $o_i$. 
This approach enables efficient training through access to global dynamics, while allowing for flexible decentralized execution. 
In this work, we employ the CTDE framework for agent training.

\subsection{Preference Formulation in Multi-Agent Settings}

In single-agent settings, following prior work \cite{pbrl_basic2012, pbrl_basic2017}, preferences can be defined as a tuple $(\sigma^{(0)}, \sigma^{(1)}, p)$. 
Here, $\sigma^{(i)}$ ($i \in \{0,1\}$) represents an observation-action sequence, also referred to as a segment, expressed as $\{s_k, a_k, \dots, s_{k+H-1}, a_{k+H-1}\}$ and \revise{share the same fixed length $H$}. 
$p \in \{0, 0.5, 1\}$ denotes the preference, indicating which of the two segments is preferred. 
Specifically, $p=0$ indicates that $\sigma^{(0)}$ is preferred over $\sigma^{(1)}$ (denoted as $\sigma^{(0)} \succ \sigma^{(1)}$), while $p=1$ indicates the opposite ($\sigma^{(1)} \succ \sigma^{(0)}$). 
$p=0.5$ indicates indifference, meaning neither segment is preferred over the other.

In multi-agent settings, we define the preference of agent $i$ as a tuple $(\sigma^{(0)}, \sigma^{(1)}, p^i)$, where $p^i$ is a binary indicator of the preferred segment for agent $i$. 
We denote $p^i=0$ as $\sigma^{(0)} \succ^i \sigma^{(1)}$.
This preference for agent $i$ is provided by a designated expert (also referred to as the teacher in this paper), as evaluating the joint behavior of multiple agents can be challenging for humans. 
Instead, these teachers can provide insights on an individual agent's performance by evaluating its adherence to the agent's role, collaboration with other agents, and contributions to the global objective. 

To align the problem formulation with the MARL framework, we define a global reward model $\hat{r}: \mathcal{S}\times \mathcal{A} \rightarrow \mathbb{R}$. 
Our algorithm aims to find a policy $\pi^i(a^i|o^i)$ for each agent $i$ that maximizes the discounted return of the reward model $\hat{r}$.
However, a gap remains between the global reward model and each agent's preference, which we will address in the subsequent section with our proposed method.

\section{\method: A Preference-based MARL Method with Agent-Specific Preference}
\label{sec:method}

\begin{algorithm}[t]
\caption{Nash equilibrium policy via iterative local preference optimization}
\label{alg:1_nash_iteration}
\begin{algorithmic}[1]
\STATE Initialize joint policy $\pi = (\pi^1, \dots, \pi^n)$
\STATE Initialize a boolean flag $\text{updated} \gets \text{true}$
\WHILE{$\text{updated} = \text{true}$}
    \STATE Set the boolean flag $\text{updated} \gets \text{false}$
    \FOR{each agent $i \in \{1, \dots, n\}$}
        \STATE Let $\Pi^i$ be the set of possible policies for agent $i$
        \FOR{each candidate policy ${\pi^i}' \in \Pi^i$}
            \STATE Generate segments $\sigma, \sigma'$ from $\pi, ~ ({\pi^i}', \pi^{-i})$
            \IF{the teacher for agent $i$ judges $\sigma' \succ^i \sigma$}
                \STATE $\pi^i \gets {\pi^i}', \; \text{updated} \gets \text{true}$
                \STATE \textbf{break} (inner for loop)
            \ENDIF
        \ENDFOR
    \ENDFOR
\ENDWHILE
\RETURN $\pi^* = (\pi^1, \dots, \pi^n)$
\end{algorithmic}
\end{algorithm}

\subsection{Theoretical Analysis}

In this subsection, we establish the theoretical foundations of the \method framework, demonstrating how our approach ensures convergence to a Nash equilibrium policy. 
\revise{Figure \ref{fig:theory_framework} illustrates the logical progression of our analysis.}

First, we assume the presence of preferences over pairs of trajectory segments for each agent. 
Specifically, for each agent $i$, there exists a teacher
\footnote{\revise{The term ``teacher'' in this paper does not strictly imply a human evaluator. In practice, teachers can be automated evaluation scripts or heuristic rules, making the framework scalable without requiring continuous human intervention.}
}
who can compare any two trajectory segments of arbitrary finite length $H$ and judge which one better reflects the agent's individual contribution to the collective goal. 
To formalize this, we introduce the following assumption.

\begin{figure}
    \centering
    \includegraphics[width=0.7\linewidth]{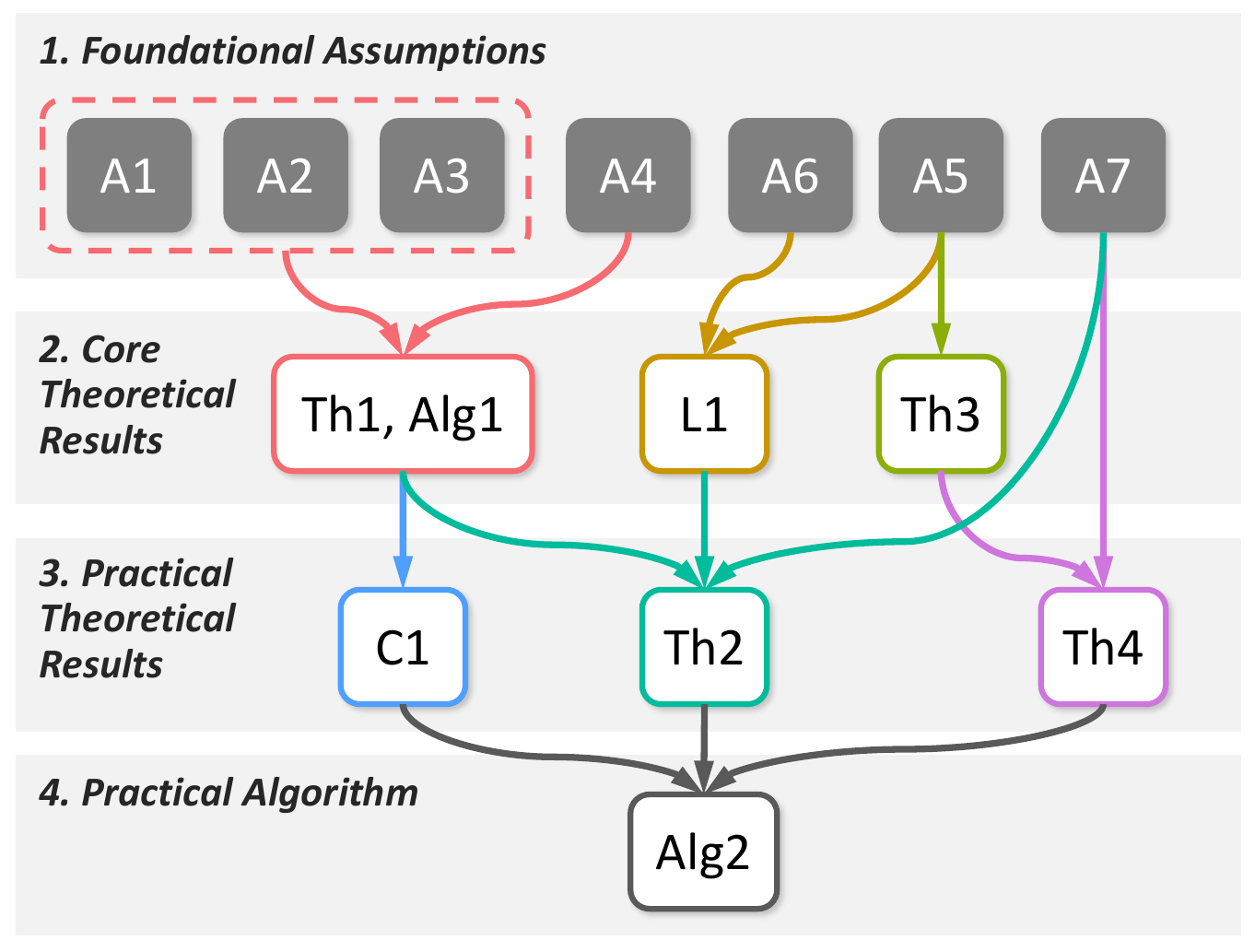}
    \caption{
    \revise{Dependency graph of assumptions (A), lemmas (L), theorems (Th), corollaries (C), and algorithms (Alg) in Section \ref{sec:method}. }
    }
    \label{fig:theory_framework}
\end{figure}

\begin{assumption} \label{asp:1_preference}
For each agent $i$, there exists a well-defined preference $p^i \in \{0, 0.5, 1\}$ over any pair of \revise{equal-length trajectory segments $(\sigma_0, \sigma_1)$ for any finite segment length $H$}. 
These preferences satisfy the properties of \textbf{symmetry}, \textbf{consistency}, and \textbf{transitivity}, which are defined as follows.
\end{assumption}

\begin{definition}[Symmetry]  
Symmetry means that if trajectory segment $\sigma_0$ is preferred over $\sigma_1$ for agent $i$, then the opposite relation also holds. Formally,
\begin{equation}
\sigma_0 \succ^i \sigma_1 \implies \sigma_1 \prec^i \sigma_0 .
\end{equation}
This ensures that preferences are reversible. 
\end{definition}

\begin{definition}[Consistency]
Consistency requires that the preference between two fixed state–action sequences is invariant to their starting time. 
Formally, if $\sigma_0^{t_0} \succ^i \sigma_1^{t_0}$, then for all $t > 0$, $\sigma_0^{t} \succ^i \sigma_1^{t}$, where $\sigma^{t}$ denotes a segment starting at time $t$, and $\sigma_i^{t_0}, \sigma_i^{t}$ ($i = 0, 1$) are with identical state-action sequence $(s,A,s',\cdots)$.
\end{definition}

\begin{definition}[Transitivity]
Transitivity guarantees logical coherence across multiple comparisons. 
Formally, if $\sigma_0 \succ^i \sigma_1$ and $\sigma_1 \succ^i \sigma_2$, then $\sigma_0 \succ^i \sigma_2$. 
This ensures that the teacher's feedback does not contradict itself when extended to multiple comparisons. 
\end{definition}

The above decentralized design, where each agent has its own evaluator, avoids the need for a single expert to assess the joint behavior of all agents. 
This makes the proposed framework applicable to heterogeneous multi-agent systems. 
We further assume that a reliable teacher is available for each agent to provide such feedback.

\begin{assumption} \label{asp:teacher}
For each agent $i$, there exists a teacher that can provide the preference feedback specified in Assumption \ref{asp:1_preference}. 
\end{assumption}

\revise{Assumptions \ref{asp:1_preference} and \ref{asp:teacher} assume that each agent's preferences are well‑defined, consistent, and readily accessible. }
This implies the teacher's feedback is based on stable criteria related to the task objectives. 
It is therefore reasonable to postulate that the task has an underlying true reward function, which reflects the global objective. 
We formalize it in Assumption \ref{asp:2_reward_bounded}.

\begin{assumption} \label{asp:2_reward_bounded}
For the multi-agent task under consideration, there exists a true global reward function $r(s, A)$.
Furthermore, its absolute value is bounded by a constant $r_{\text{max}}$, i.e., $|r(s, A)| \le r_{\text{max}}$ for all state-action pairs.
\end{assumption}

The bounded reward in Assumption \ref{asp:2_reward_bounded} is a common and practical condition, which prevents divergence in learning and supports theoretical analysis \cite{melo2001convergence, li2021decentralized, li2024anocbabased}. 
Then, we establish the connection between the agent's local preferences and the global objective through the following assumption.

\begin{assumption} \label{asp:3_alignment}
For any agent $i$, consider a joint policy $\pi = (\pi^i, \pi^{-i})$ and an alternative policy ${\pi^i}'$ for agent $i$ with $\pi^{-i}$ fixed. 
Let $\sigma$ and $\sigma'$ be trajectory segments generated by $\pi$ and $({\pi^i}', \pi^{-i})$, respectively. 
The local preference aligns with the global return such that $\sigma' \succ^i \sigma$ if and only if the expected global return under $({\pi^i}', \pi^{-i})$ is strictly greater than that under $\pi$. 
Formally, 
\begin{equation}
\begin{aligned}
& \mathbb{E}_{\tau \sim (\pi^{i'}, \pi^{-i})} \left[ \sum_{t} \gamma^t r(s_t, A_t) \right] > 
\\
& \quad\quad\quad\quad
\mathbb{E}_{\tau \sim (\pi^{i}, \pi^{-i})} \left[ \sum_{t} \gamma^t r(s_t, A_t) \right] 
\Longleftrightarrow \sigma' \succ^i \sigma .
\end{aligned}
\end{equation}
\end{assumption}

Assumption \ref{asp:3_alignment} establishes that an agent policy's update favored by the agent's teacher, i.e., leading to a locally preferred trajectory, corresponds strictly to an improvement in the true global objective. 
\revise{Intuitively, this means a teacher's local approval of an agent's behavior never contradicts the global team objective. }
This alignment ensures that optimizing agent-specific preferences will drive the system toward globally optimal joint policies, thereby enabling the convergence results presented as follows.

Under Assumptions \ref{asp:1_preference}, \ref{asp:teacher}, \ref{asp:2_reward_bounded} and \ref{asp:3_alignment}, we now prove that a straightforward iterative algorithm converges to a Nash equilibrium joint policy. 
Algorithm \ref{alg:1_nash_iteration} formalizes this process, which systematically examines each agent's policy space while holding the remaining agents fixed, and updates agent policies based exclusively on agent-specific preferences.

\begin{theorem} \label{thm:1_nash_iteration}
    We consider the case where state and action spaces are finite and countable. 
    Under Assumptions \ref{asp:1_preference}, \ref{asp:teacher}, \ref{asp:2_reward_bounded}, and \ref{asp:3_alignment}, the joint policy $\pi^* = (\pi^{1*}, \dots, \pi^{n*})$ returned by Algorithm \ref{alg:1_nash_iteration} constitutes a Nash equilibrium when the trajectory segment length $H \to \infty$. 
    Specifically, for any agent $i$ and any alternative policy ${\pi^i}' \in \Pi^i$,
    \begin{equation}
    \begin{aligned}
        & \mathbb{E}_{\tau \sim (\pi^{i*}, \pi^{-i*})} \left[ \sum_{t=0}^{\infty} \gamma^t r(s_t, A_t) \right] \ge 
        \\
        & \quad\quad\quad\quad
        \mathbb{E}_{\tau \sim ({\pi^i}', \pi^{-i*})} \left[ \sum_{t=0}^{\infty} \gamma^t r(s_t, A_t) \right].
    \end{aligned}
    \end{equation}
\end{theorem}

\begin{proof}
    We first show that Algorithm \ref{alg:1_nash_iteration} terminates under the given condition. 
    \revise{Since we consider finite state and action spaces, the existence of at least one Nash equilibrium is guaranteed by Nash's theorem \cite{nash1950equilibrium}. }
    Also, the finiteness of state and action spaces implies that each agent's policy space and the joint policy space are finite. 
    In each iteration, Algorithm \ref{alg:1_nash_iteration} checks each agent $i$ and updates its policy if an alternative policy $\pi^{i'}$ exists such that the teacher prefers segments generated by $(\pi^{i'}, \pi^{-i})$ over those generated by the current policy. 
    By Assumption \ref{asp:3_alignment}, such an update strictly increases the expected global return. 
    Since the global return is bounded above (Assumption \ref{asp:2_reward_bounded}) and the joint policy space is finite, the algorithm cannot update indefinitely without repeating a policy. 
    Furthermore, because each update increases the global return, it cannot cycle. 
    Therefore, Algorithm \ref{alg:1_nash_iteration} terminates after a finite number of iterations.

    Suppose Algorithm \ref{alg:1_nash_iteration} returns a joint policy $\pi^*$. 
    We prove that $\pi^*$ constitutes a Nash equilibrium by contradiction. 
    Assume $\pi^*$ is not a Nash equilibrium. 
    Then, there exists an agent $i$ and an alternative policy ${\pi^i}' \in \Pi^i$ such that, while keeping the policies of other agents $\pi^{-i*}$ fixed, the expected global return under the altered joint policy $({\pi^i}', \pi^{-i*})$ is strictly greater than that under $\pi^*$. 
    Formally, 
    { \small
    \begin{equation} \label{eq:global_improvement}
        \mathbb{E}_{\tau' \sim ({\pi^i}', \pi^{-i*})} \left[ \sum_{t=0}^{\infty} \gamma^t r(s_t, A_t) \right] > \mathbb{E}_{\tau \sim \pi^*} \left[ \sum_{t=0}^{\infty} \gamma^t r(s_t, A_t) \right].
    \end{equation}
    }
    
    By Assumption \ref{asp:3_alignment}, the inequality in global return implies a corresponding preference from the agent $i$'s teacher. 
    Let $\sigma^*$ and $\sigma'$ be trajectory segments of length $H$ sampled from the infinite trajectories $\tau$ and $\tau'$ generated by $\pi^*$ and $({\pi^i}', \pi^{-i*})$, respectively. 
    For sufficiently long $H$, the difference in their truncated returns will reflect the difference in their infinite-horizon returns. 
    
    More precisely, due to the bounded reward (Assumption \ref{asp:2_reward_bounded}), the tail contribution of the return beyond time $H$ is bounded:
    \begin{equation} \label{eq:tail_bound}
        \left| \sum_{t=H}^{\infty} \gamma^t r(s_t, A_t) \right| \le \frac{\gamma^H}{1-\gamma} r_{\text{max}}.
    \end{equation}
    
    Let $R_H(\sigma) = \sum_{t=0}^{H-1} \gamma^t r(s_t, A_t)$ denote the discounted return over a segment $\sigma$ of length $H$. The difference in infinite-horizon returns can be expressed as:
    \begin{equation}
    \begin{aligned}
        \Delta R_{\infty} & = \mathbb{E} \left[ \sum_{t=0}^{\infty} \gamma^t r(s_t', A_t') - \sum_{t=0}^{\infty} \gamma^t r(s_t^*, A_t^*) \right] \\
        & = \underbrace{\mathbb{E} \left[ R_H(\sigma') - R_H(\sigma^*) \right]}_{\Delta R_H} 
        \\
        & \quad + \underbrace{\mathbb{E} \left[ \sum_{t=H}^{\infty} \gamma^t \left( r(s_t', A_t') - r(s_t^*, A_t^*) \right) \right]}_{\Delta R_{\text{tail}}}.
    \end{aligned}
    \end{equation}
    
    From \eqref{eq:global_improvement}, we have $\Delta R_{\infty} > 0$. The tail difference $\Delta R_{\text{tail}}$ is bounded by $2 \frac{\gamma^H}{1-\gamma} r_{\text{max}}$ due to \eqref{eq:tail_bound}. Therefore, for the inequality $\Delta R_{\infty} > 0$ to hold, the segment return difference $\Delta R_H$ should be positive and sufficiently large to overcome the maximum possible negative contribution from the tails. Specifically, a sufficient condition is:
    \begin{equation} \label{eq:segment_diff_sufficient}
        \Delta R_H > 2 \frac{\gamma^H}{1-\gamma} r_{\text{max}}.
    \end{equation}
    
    As $H \to \infty$, the right-hand side of \eqref{eq:segment_diff_sufficient} converges to zero. This implies that for any $\Delta R_{\infty} > 0$, there exists a finite $H_0$ such that for all $H \ge H_0$, the generated segments $\sigma'$ and $\sigma^*$ will satisfy $\Delta R_H > 0$. According to Assumption \ref{asp:3_alignment}, this strictly positive difference in segment returns implies the teacher's preference $\sigma' \succ^i \sigma^*$.
    
    However, Algorithm \ref{alg:1_nash_iteration} terminates only when no agent $i$ can find a policy ${\pi^i}'$, such that the newly generated segment is preferred by its teacher over the segment from the current policy $\pi^i$. 
    The existence of such a policy ${\pi^i}'$ for agent $i$, as established above, contradicts the termination condition of the algorithm. 
    Thus, our initial assumption must be false, and the joint policy $\pi^*$ must be a Nash equilibrium.
    This completes the proof.
\end{proof}

\begin{remark}
Theorem \ref{thm:1_nash_iteration} provides a theoretical existence proof. 
It shows that, in principle, optimizing decentralized, agent-specific preferences can lead the system to a global Nash equilibrium. 
However, its assumptions are idealized and computationally impractical, such as infinite trajectory lengths and querying teachers for every policy update, and Algorithm \ref{alg:1_nash_iteration} serves primarily as a conceptual framework. 
The practical implementation of our approach, which employs learned reward models under finite preference data, is detailed in Section \ref{subsec:reward_modeling}, \ref{subsec:marl_with_learned_rewards} and supported by Theorems \ref{thm:2_sum_nash}, \ref{thm:4_epsilon_nash}. 
\end{remark}

\revise{Theorem \ref{thm:1_nash_iteration} states that, if agents keep upgrading based on their own teacher's feedback, the team eventually reaches a joint policy where no one can unilaterally improve, and each agent finds their best role in the team. }
The proof establishes that under infinite segment length, local preferences perfectly reflect global return improvements, ensuring Algorithm \ref{alg:1_nash_iteration} terminates only when no beneficial unilateral deviation exists. 
In practical implementations, we typically select pairs of segments with distinctly different behaviors for comparison \cite{mu2024sepoa, mu2025clarify}, aiming to facilitate reliable teacher judgments.  
Therefore, it is reasonable to assume a minimum return difference $\delta > 0$ between any two segments that the teacher can distinguish, that is, $\exists \delta \ge 0$ such that $|R_H(\sigma_0) - R_H(\sigma_1)| \ge \delta > 0$, where $R_H(\sigma) = \sum_{t=0}^{H-1} \gamma^t r(s_t, A_t)$ denote the discounted return over a segment $\sigma$ of length $H$. 
Under this assumption, the following corollary provides a sufficient condition on $H$ to guarantee correctness when trajectory segments exhibit sufficient discriminability.

\begin{corollary} \label{cor:1}
    If all segment pairs are sufficiently distinct, i.e. $\exists \delta > 0$ with $|R_H(\sigma_0) - R_H(\sigma_1)| \ge \delta$ for any $\sigma_0, \sigma_1$, 
    then the policy $\pi^*$ returned by Algorithm \ref{alg:1_nash_iteration} is a Nash equilibrium when the segment length satisfies $H \ge \log_\gamma \frac{\delta(1-\gamma)}{2r_{\text{max}}}$.
\end{corollary}

\revise{This corollary follows directly from the proof of Theorem \ref{thm:1_nash_iteration}, which ensures that the finite-horizon, sufficiently long segment captures enough of the return difference to dominate the infinite tail bound. }
Notably, our algorithm does not require a fixed $H$ for policy convergence. 
Its effectiveness, established in Theorem \ref{thm:2_sum_nash} and \ref{thm:4_epsilon_nash}, depends on learned reward models aligned with teachers' preferences, as guaranteed by Theorem \ref{thm:reward_consistency}, without imposing constraints on the segment length.

Building on the convergence guarantee in Theorem \ref{thm:1_nash_iteration}, we formalize the connection between agent-specific preferences and the reward modeling in \method. 
To achieve this, we assume the existence of local reward functions that reflect each agent's preferences. 
Given that teacher feedback is stable and consistent (Assumption \ref{asp:1_preference}), it is reasonable to assume that each agent has a well-defined task objective, and its behavior can be evaluated by a corresponding reward function.
Specifically, in Assumption \ref{asp:4_local_reward_alignment}, we posit that each agent $i$ has a latent local reward function $r^i(s, A)$ that quantifies the agent's individual contribution to the global task. 
This local reward is defined over the global state and joint action, as an agent's contribution depends on the context of the overall system.

\begin{assumption} \label{asp:4_local_reward_alignment}
For each agent $i$, there exists a local reward function $r^i(s, A)$ such that the teacher's preference over trajectory segments is consistent with the discounted sum of this reward. 
Formally, for any two segments $\sigma^{(0)}$ and $\sigma^{(1)}$ of equal length $H$, 
\begin{equation}
\sigma^{(1)} \succ^i \sigma^{(0)} \; \Longleftrightarrow \; R^i(\sigma^{(1)}) > R^i(\sigma^{(0)}),
\end{equation}
where $R^i(\sigma) = \sum_{(s_t, A_t) \in \sigma} \gamma^t r^i(s_t, A_t)$ denotes the discounted local return for agent $i$ over segment $\sigma$. 
\end{assumption}

\revise{This assumption states that preferences provided by the teacher for agent $i$ are consistent with an implicit local reward signal. }
Such an assumption is reasonable, since it has been widely used in single-agent preference-based RL \cite{pebble, surf}, where preferences are often assumed to correspond to an implicit reward model. 
Furthermore, we assume the global reward function can be expressed as a function of these local rewards. 
\revise{This captures that the team's overall performance arises from the contributions of individual agents. }

\begin{assumption} \label{asp:5_global_reward_composition}
The true global reward function $r(s, A)$ can be expressed as a function $F$ of the local rewards:
\begin{equation}
r(s, A) = F\left(r^1(s, A), \dots, r^n(s, A)\right),
\end{equation}
where $F: \mathbb{R}^n \to \mathbb{R}$ is a continuously differentiable function. 
\end{assumption}

While the exact form of $F$ does not need to be known a priori, it satisfies certain properties. 
Specifically, Assumption \ref{asp:3_alignment} requires that an improvement in any agent's local reward does not reduce the global return. 
A sufficient condition to satisfy this requirement is that $F$ is monotonically increasing with respect to each local reward, as formalized below.

\begin{lemma} \label{lem:monotonicity}

If the function $F$ is strictly increasing in each argument (i.e., $\partial F / \partial r^i > 0$ for all $i \in \{1, \ldots, n\}$), then under Assumptions \ref{asp:3_alignment}, \ref{asp:4_local_reward_alignment} and \ref{asp:5_global_reward_composition}, any trajectory segment $\sigma'$ that is locally preferred over $\sigma$ yields a higher expected global return.
\end{lemma}

\begin{proof}
Consider a joint policy $\pi = (\pi^i,\pi^{-i})$ and an alternative policy $\pi^{i'}$ for agent $i$. 
Let $\sigma$ and $\sigma'$ be segments generated by $\pi$ and $(\pi^{i'},\pi^{-i})$, respectively. 
If the update leads to a locally preferred segment, i.e., $\sigma' \succ^i \sigma$, then Assumption \ref{asp:4_local_reward_alignment} implies an increase in the local return: $R^i(\sigma') > R^i(\sigma)$, where $R^i(\sigma), R^i(\sigma')$ is the discounted local return of agent $i$ for segments $\sigma$ and $\sigma'$. 
Because $F$ is strictly increasing in each argument, a higher local return $R^i(\sigma')$ contributes positively to the global return $R(\sigma') = \sum_t \gamma^t F\big(r^1(s_t, A_t),\dots,r^n(s_t, A_t)\big)$. 
Under the continuity and differentiability of $F$ (Assumption \ref{asp:5_global_reward_composition}), this guarantees that the expected global return under $(\pi^{i'},\pi^{-i})$ is strictly greater than under $\pi$, which is exactly the  Assumption \ref{asp:3_alignment}. 
Thus, monotonicity of $F$ is a sufficient condition to align local preferences with the global objective, which completes the proof. 
\end{proof}

Lemma \ref{lem:monotonicity} supports using a monotonic function to aggregate local rewards into a global reward signal, which aligns naturally with fully cooperative settings where the global reward should increase as local contributions improve. 
\revise{Intuitively, Lemma \ref{lem:monotonicity} means that as long as the team performance increases whenever any individual contribution increases, chasing individual improvement automatically chases team success. }
While Lemma \ref{lem:monotonicity} characterizes the structure of $F$, estimating $F$ directly remains challenging. 
We make a practical simplification by considering the additive composition $F(r^1, \ldots, r^n) = \sum_{i=1}^n r^i$, which naturally satisfies the monotonicity condition. 
\revise{Importantly, this additive aggregation is a practical instantiation of the general function $F$ in Assumption \ref{asp:5_global_reward_composition} and thus preserves all theoretical properties. }
To guarantee that optimizing the sum of local rewards leads to a Nash equilibrium, we assume that the negative impact on other agents does not overwhelm the positive gain of the improving agent, as formalized below.

\begin{assumption} \label{asp:6_bounded_cross_impact}
Consider a joint policy $\pi = (\pi^i, \pi^{-i})$ and an alternative policy ${\pi^i}'$ for agent $i$ that yields a locally preferred segment $\sigma' \succ^i \sigma$ and improves the local return $R^i(\sigma') > R^i(\sigma)$. 
The resulting change in the sum of local rewards for all other agents is bounded from below:
\begin{equation}
R^{-i}(\sigma') - R^{-i}(\sigma) > - \left[ R^i(\sigma') - R^i(\sigma) \right],
\end{equation}
where $R^{-i}(\sigma) = \sum_{j \neq i} R^j(\sigma)$. 
This implies the aggregate local return for all agents does not decrease: $\sum_{j=1}^n R^j(\sigma') > \sum_{j=1}^n R^j(\sigma)$. 
\end{assumption}

\revise{Assumption \ref{asp:6_bounded_cross_impact} guarantees that any locally preferred policy update for one agent will not degrade the team's overall objective. }
This is a reasonable condition for cooperative settings, where agents work toward a shared goal. 
In this case, actions that significantly harm teammates are unlikely to be preferred by specialized teachers that value collaboration. 
In our practical algorithm, we incorporate a regularization term to encourage consistency between local and aggregated global preferences, which promotes this condition. 
The implementation details are discussed in Section \ref{subsec:reward_modeling}.

Based on the assumptions above, we now present the theoretical result that connects the practical algorithm, which sums the learned local reward models, to the theoretical guarantee of achieving a Nash equilibrium.

\begin{theorem} \label{thm:2_sum_nash}
Consider a multi-agent system satisfying Assumptions \ref{asp:1_preference} to \ref{asp:6_bounded_cross_impact}. 
Suppose we learn local reward functions $\hat{r}^i(s, A)$ that are perfectly aligned with the teachers' preferences, i.e., satisfy Assumption \ref{asp:4_local_reward_alignment}. 
Let the estimated global reward be the sum of local rewards: $\hat{r}(s, A) = \sum_{i=1}^n \hat{r}^i(s, A)$. 
Then, any joint policy $\pi^*$ that maximizes the expected discounted return under $\hat{r}$, i.e., 
\begin{equation}
\pi^* \in \arg\max_{\pi} \mathbb{E}_{\tau \sim \pi} \left[ \sum_{t=0}^{\infty} \gamma^t \hat{r}(s_t, A_t) \right],
\end{equation}
constitutes a Nash equilibrium for the true global reward $r(s, A)$.
\end{theorem}

\begin{proof}
We prove by contradiction. 
Assume $\pi^*$ is not a Nash equilibrium. 
Then, there exists an agent $i$ and an alternative policy ${\pi^i}'$ such that, holding $\pi^{-i*}$ fixed, the altered joint policy $\pi' = ({\pi^i}', \pi^{-i*})$ achieves a strictly higher expected global return under the true reward $r$:
\begin{equation} \label{eq:true_improvement}
\mathbb{E}_{\tau' \sim \pi'} \left[ \sum_{t} \gamma^t r(s_t, A_t) \right] > \mathbb{E}_{\tau \sim \pi^*} \left[ \sum_{t} \gamma^t r(s_t, A_t) \right].
\end{equation}

By Assumption \ref{asp:3_alignment}, this improvement in global return implies that for trajectory segments $\sigma^*$ and $\sigma'$ generated by $\pi^*$ and $\pi'$ respectively, we have $\sigma' \succ^i \sigma^*$. 
According to Assumption \ref{asp:4_local_reward_alignment}, this preference implies an improvement in the local return for agent $i$: $R^i(\sigma') > R^i(\sigma^*)$.

Now, consider the estimated global reward $\hat{r} = \sum_{j} \hat{r}^j$. 
The expected return under $\hat{r}$ for policy $\pi'$ and $\pi^*$ can be expressed in terms of their local returns. 
From Assumption \ref{asp:6_bounded_cross_impact}, the improvement in agent $i$'s local return, combined with the bounded negative impact on others, guarantees that the sum of local returns increases:
\begin{equation}
\sum_{j=1}^n R^j(\sigma') > \sum_{j=1}^n R^j(\sigma^*).
\end{equation}
Since the local reward models $\hat{r}^j$ perfectly reflect the teachers' preferences, the discounted sum of $\hat{r}^j$ over a segment equals $R^j(\sigma)$ for that segment. 
Therefore, the inequality above implies that for the segments $\sigma'$ and $\sigma^*$, we have:
\begin{equation}
\sum_{t} \gamma^t \hat{r}(s'_t, A'_t) > \sum_{t} \gamma^t \hat{r}(s^*_t, A^*_t).
\end{equation}
Taking expectation over time steps, we obtain the corresponding inequality for the infinite-horizon expected return: 
\begin{equation}
\mathbb{E}_{\tau' \sim \pi'} \left[ \sum_{t} \gamma^t \hat{r}(s_t, A_t) \right] > \mathbb{E}_{\tau \sim \pi^*} \left[ \sum_{t} \gamma^t \hat{r}(s_t, A_t) \right].
\end{equation}
This contradicts the premise that $\pi^*$ maximizes the expected return under $\hat{r}$. Hence, $\pi^*$ must be a Nash equilibrium, which completes the proof.
\end{proof}

Theorem \ref{thm:2_sum_nash} provides the theoretical foundation for the \method algorithm. 
\revise{It establishes that under mild assumptions, maximizing the global reward, which is computed by summing all the perfectly learned local reward models, is equivalent to finding a Nash equilibrium policy. }
This result bridges our practical algorithm with the game-theoretic solution concept, ensuring the effectiveness of our preference-based optimization with decentralized feedback. 
Based on these results, we will describe the detailed \method algorithm in the next subsection, showing how this framework can be applied to optimize multi-agent policies in practical scenarios.

\subsection{Multi-Agent Reward Modeling}
\label{subsec:reward_modeling}

Based on the theoretical foundations established in the previous section, we now focus on the practical implementation of the \method framework. 
In particular, we construct an agent-specific reward model for each agent that aligns with the preferences provided by its dedicated teacher, and then aggregate these decentralized, agent-wise reward models into a global signal suitable for centralized training.

Inspired by prior work in single-agent PbRL \cite{pbrl_basic2017}, we define a preference predictor for each agent $i$, denoted as $P_{\psi}[\sigma^{(0)} \succ^i \sigma^{(1)}]$, which predicts the probability that segment $\sigma^{(1)}$ is preferred over $\sigma^{(0)}$ by the teacher of agent $i$. 
This predictor is parameterized by $\psi$ and is trained by minimizing the cross-entropy loss:
{\small
\begin{equation}
\label{eq:ce_loss}
\begin{aligned}
\mathcal{L}_{\text{pref}}^i = -\mathbb{E}_{(\sigma^{(0)},\sigma^{(1)},p^i) \sim \mathcal{D}} \Big[ & p^i(0) \log P_{\psi}[\sigma^{(0)} \succ^i \sigma^{(1)}] \\
& + p^i(1) \log P_{\psi}[\sigma^{(1)} \succ^i \sigma^{(0)}] \Big],
\end{aligned}
\end{equation}
}
where $p^i(0)$ and $p^i(1)$ are indicator functions for the preference label. 
Using the Bradley-Terry model \cite{bradley-terry, pbrl_basic2017}, we formulate preference probability based on the agent-specific reward model $\hat{r}_\psi^i$. 
For a segment pair $(\sigma^{(0)}, \sigma^{(1)})$, where $\sigma^{(j)} = \{ s_0^{(j)}, A_0^{(j)}, \dots, s_{H-1}^{(j)}, A_{H-1}^{(j)} \}$ with $j \in \{0,1\}$, the probability that segment $\sigma^{(1)}$ is preferred over $\sigma^{(0)}$ by the teacher assigned to agent $i$ is expressed as: 
{\small
\begin{equation}
\label{eq:bt_model}
P_{\psi}[\sigma^{(1)} \succ^i \sigma^{(0)}] = \frac{\exp\left( \sum_{t} \gamma^t \hat{r}_\psi^i(s_t^{(1)}, A_t^{(1)}) \right)}{\sum_{j \in \{0,1\}} \exp\left( \sum_{t} \gamma^t \hat{r}_\psi^i(s_t^{(j)}, A_t^{(j)}) \right)}.
\end{equation}
}
This equation establishes a probabilistic relationship where the preference probability is exponentially related to the discounted sum of the learned reward over the segment. 
The reward model $\hat{r}_\psi^i$ is trained to maximize the likelihood of the observed preferences, thereby aligning it with the teacher's judgments.

% ============

We provide a theoretical guarantee that under sufficient preference data with adequate coverage, optimizing the cross-entropy loss in \eqref{eq:ce_loss} yields a reward model $\hat{r}_\psi^i$ that aligns with the underlying teacher preferences. 
This result formalizes the sample complexity requirement for reward learning in \method.

\begin{theorem}
\label{thm:reward_consistency}
Let Assumption \ref{asp:4_local_reward_alignment} hold, i.e., there exists a true local reward function $r_i(s, A)$ that reflects the teacher's preferences for agent $i$. 
Consider a preference dataset $\mathcal{D}_i = \{(\sigma_k^{(0)}, \sigma_k^{(1)}, p_k^i)\}_{k=1}^N$ sampled independently from a distribution $\mu$ over segment pairs, where $\mu$ has full support on the relevant state-action subspace $\Omega_i \subseteq \mathcal{S} \times \mathcal{A}$. 
Let $\mathcal{F}$ be the function class of reward models $\hat{r}_\psi^i$ with bounded norm $\|\hat{r}_\psi^i\|_\infty \leq r_{\max}$, and assume $r_i \in \mathcal{F}$. 
Define the empirical risk minimizer:
\begin{equation}
\label{eq:erm}
\hat{r}_{\psi^*}^i \in \arg\min_{\hat{r}_\psi^i \in \mathcal{F}} \mathcal{L}_{\text{pref}}^i(\hat{r}_\psi^i; \mathcal{D}_i).
\end{equation}

Then, as the dataset size $N \to \infty$, $\hat{r}_{\psi^*}^i$ converges uniformly to $r_i$ on $\Omega_i$ up to an affine transformation. 
Specifically, for any $\epsilon > 0$, 
{ \footnotesize
\begin{equation}
\lim_{N \to \infty} \Pr\Big( \inf_{\alpha>0,\beta\in\mathbb{R}} \sup_{(s,A)\in\Omega_i} \big| \hat{r}_{\psi^*}^i(s,A) - (\alpha r_i(s,A) + \beta) \big| > \epsilon \Big) = 0.
\end{equation}
}

Furthermore, for finite $N$, with probability at least $1 - \delta$, the estimation error is bounded by:
\begin{equation}
\|\hat{r}_{\psi^*}^i - r_i\|_{L_2(\mu)}^2 \leq O\!\left( \mathcal{R}_N(\mathcal{F}) + \sqrt{\frac{\log(1/\delta)}{N}} \right),
\end{equation}
where $\mathcal{R}_N(\mathcal{F})$ denotes the Rademacher complexity of $\mathcal{F}$. 
\end{theorem}

\begin{proof}
First, we show that with infinite data, $r_i$ is identifiable up to an affine transformation. 
Define the expected cross-entropy loss under $\mu$:
{\small
\begin{equation}
\begin{aligned}
\mathcal{L}_\mu(\hat{r}) = \mathbb{E}_{(\sigma^{(0)},\sigma^{(1)},p^i)\sim\mu} 
\bigg[ & -p^i(1)\log P_{\hat{r}}[\sigma^{(1)}\succ^i\sigma^{(0)}] 
\\
& - p^i(0)\log P_{\hat{r}}[\sigma^{(0)}\succ^i\sigma^{(1)}] 
\bigg],
\end{aligned}
\end{equation}
}
where $P_{\hat{r}}$ uses $\hat{r}$ in \eqref{eq:bt_model}. 
Let $\sigma^{(j)}_{0:H-1}$ denote the segment $\{s_t^{(j)}, A_t^{(j)}\}_{t=0}^{H-1}$. 
Using the Bradley-Terry model, we have:
\begin{align}
P_{\hat{r}}[\sigma^{(1)}\succ^i\sigma^{(0)}] = \frac{\exp(R_i(\sigma^{(1)}))}{\exp(R_i(\sigma^{(0)})) + \exp(R_i(\sigma^{(1)}))},
\end{align}
with $R_i(\sigma) = \sum_{t=0}^{H-1}\gamma^t\hat{r}(s_t,A_t)$. 
Under Assumption \ref{asp:4_local_reward_alignment}, the true preference probability is:
\begin{align}
P_{r_i}[\sigma^{(1)}\succ^i\sigma^{(0)}] = \frac{\exp(R_i^*(\sigma^{(1)}))}{\exp(R_i^*(\sigma^{(0)})) + \exp(R_i^*(\sigma^{(1)}))},
\end{align}
where $R_i^*$ uses $r_i$. 
The cross-entropy loss is strictly proper \cite{bradley-terry}, that is:
\begin{align}
\mathcal{L}_\mu(\hat{r}) \geq \mathcal{L}_\mu(r_i), \quad \forall \hat{r}\in\mathcal{F},
\end{align}
with equality if and only if $P_{\hat{r}} = P_{r_i}$ almost surely. 
Since the logistic function is strictly monotonic, $P_{\hat{r}} = P_{r_i}$ implies $R_i(\sigma) = R_i^*(\sigma) + C(\sigma)$ for some constant $C(\sigma)$ independent of $\sigma$'s content, which reduces to $\hat{r}(s,A) = \alpha r_i(s,A) + \beta$ for some $\alpha>0,\beta\in\mathbb{R}$. 
Thus, the minimizer of $\mathcal{L}_\mu$ is unique up to an affine transformation.

Then, we prove the convergence of the loss minimizer.
The empirical loss $\mathcal{L}_{\text{pref}}^i (\hat{r};\mathcal{D}_i)$ is the sample average of i.i.d. terms. 
By the uniform law of large numbers \cite{bartlett2002rademacher}, since $\mathcal{F}$ is uniformly bounded and $\mu$ has full support, we have:
\begin{equation}
\sup_{\hat{r}\in\mathcal{F}} \big| \mathcal{L}_{\text{pref}}^i(\hat{r};\mathcal{D}_i) - \mathcal{L}_\mu(\hat{r}) \big| \xrightarrow{p} 0 \quad \text{as } N\to\infty.
\end{equation}
Let $\hat{r}_{\psi^*}^i$ be the empirical minimizer and $r_i^*$ be any minimizer of $\mathcal{L}_\mu$ (i.e., $r_i^* = \alpha r_i + \beta$). 
Using the standard inequality:
\begin{align}
\mathcal{L}_\mu(\hat{r}_{\psi^*}^i) - \mathcal{L}_\mu(r_i^*) &\leq 2\sup_{\hat{r}\in\mathcal{F}} \big| \mathcal{L}_{\text{pref}}^i(\hat{r};\mathcal{D}_i) - \mathcal{L}_\mu(\hat{r}) \big|,
\end{align}
the uniform convergence implies $\mathcal{L}_\mu(\hat{r}_{\psi^*}^i) \to \mathcal{L}_\mu(r_i^*)$ in probability. 
By the strict properness established above, this yields the uniform convergence of $\hat{r}_{\psi^*}^i$ to $r_i^*$ on $\Omega_i$. 

Finally, we prove that the reward estimation error is bounded.
For any $\hat{r}\in\mathcal{F}$, define the loss difference class $\mathcal{G} = \{g_{\hat{r}} = \ell_{\hat{r}} - \ell_{r_i^*} : \hat{r}\in\mathcal{F}\}$, where $\ell_{\hat{r}}$ is the per-sample loss. 
Each $g_{\hat{r}}$ is bounded in $[0,2r_{\max}]$. 
By standard symmetrization \cite{bartlett2002rademacher}:
\begin{align}
\mathbb{E}\Big[\sup_{\hat{r}\in\mathcal{F}} \big| \mathcal{L}_{\text{pref}}^i(\hat{r};\mathcal{D}_i) - \mathcal{L}_\mu(\hat{r}) \big| \Big] \leq 2\mathcal{R}_N(\mathcal{G}).
\end{align}
Since $\mathcal{G}$ is a linear shift of $\mathcal{F}$, $\mathcal{R}_N(\mathcal{G}) = \mathcal{R}_N(\mathcal{F})$. 
Applying McDiarmid's inequality (due to bounded differences) yields with probability $1-\delta$, we have:
\begin{align}
\sup_{\hat{r}\in\mathcal{F}} \big| \mathcal{L}_{\text{pref}}^i(\hat{r};\mathcal{D}_i) - \mathcal{L}_\mu(\hat{r}) \big| \leq \mathcal{R}_N(\mathcal{F}) + \sqrt{\frac{2\log(2/\delta)}{N}}.
\end{align}
Combining this with the strong convexity of $\mathcal{L}_\mu$ in a neighborhood of $r_i^*$, which is implied by the logistic loss, and the boundedness of rewards gives the $L_2$-error bound, we have:
\begin{align}
\|\hat{r}_{\psi^*}^i - r_i^*\|_{L_2(\mu)}^2 \leq 4r_{\max}\Big(\mathcal{R}_N(\mathcal{F}) + \sqrt{\frac{2\log(2/\delta)}{N}}\Big).
\end{align}
Since $r_i^*$ differs from $r_i$ only by an affine transformation, this completes the proof.
\end{proof}

\begin{remark}
The condition that $\mu$ has full support on $\Omega_i$ ensures that the preference data covers all state-action regions relevant to agent $i$'s role. 
\revise{While it is an idealized theoretical condition, this motivates the use of continuous preference collection in practice (detailed in Section \ref{subsec:marl_with_learned_rewards}). 
By collecting preferences using segments generated by the current policy, the data adapts to regions frequently visited under evolving policies.}
\end{remark}

% ============

\revise{Theorem \ref{thm:reward_consistency} states that given enough preference comparisons, the reward model converges to the teacher's true judgment criteria. }
While Theorem \ref{thm:reward_consistency} shows convergence under infinite preference data, data is finite in practical scenarios. 
We therefore analyze how a reward model trained on a finite sample size $N$ affects the equilibrium.

\begin{theorem}
\label{thm:4_epsilon_nash}
Let $N$ be the sample size of preference data used for reward modeling. 
Following Theorem \ref{thm:reward_consistency}, with probability at least $1-\delta$, there exists a uniform error bound $\epsilon_r(N) = O(1/\sqrt{N})$ between the learned reward model $\hat{r}^i_\psi$ and the affine-transformed true latent reward $r^i$ for each agent $i$. 
That is, for all $(s, A) \in \mathcal{S} \times \mathcal{A}$:
\begin{equation}
|\hat{r}^i_\psi(s, A) - (\alpha r^i(s, A) + \beta)| \leq \epsilon_r(N),
\end{equation}
where $\alpha > 0$ and $\beta$ are scalar constants. 
Let $\pi^* = (\pi^{1*}, \dots, \pi^{n*})$ be a Nash equilibrium policy under the aggregated learned reward $\hat{r}(s, A) = \sum_{i=1}^n \hat{r}^i_\psi(s, A)$. 
Then, $\pi^*$ constitutes an $\epsilon$-Nash equilibrium under the global reward objective reflecting the true preferences, denoted as $R_{true}(s, A) = \alpha \sum_{i=1}^n r^i(s, A) + n\beta$, where:
\begin{equation}
\epsilon = \frac{2n \epsilon_r(N)}{1 - \gamma}.
\end{equation}
This implies that as the sample size $N \to \infty$, $\epsilon \to 0$, and the policy $\pi^*$ converges to an exact Nash equilibrium.
\end{theorem}

\begin{proof}
Let $\tilde{r}^i(s, A) = \alpha r^i(s, A) + \beta$ denote the proxy of the ground-truth reward targeted by the learning process, and let $\tilde{r}(s, A) = \sum_{i=1}^n \tilde{r}^i(s, A)$ be the corresponding global reward. 
We aim to prove that $\pi^*$ is an $\epsilon$-Nash equilibrium under $\tilde{r}$. 

Based on Theorem \ref{thm:reward_consistency}, the approximation error for a single agent's reward is bounded by $\epsilon_r(N)$. 
Since the global reward is the sum of local rewards, its error bound is derived as:
\begin{equation}
\begin{aligned}
\left| \hat{r}(s, A) - \tilde{r}(s, A) \right| 
& = \left| \sum_{i=1}^n \hat{r}^i_\psi(s, A) - \sum_{i=1}^n \tilde{r}^i(s, A) \right| 
\\
& \leq \sum_{i=1}^n \left| \hat{r}^i_\psi(s, A) - \tilde{r}^i(s, A) \right| 
\\
& \leq n \epsilon_r(N).
\end{aligned}
\end{equation}

Then, let $J^{\pi}_{r}$ denote the expected discounted return of a joint policy $\pi$ under a reward function $r$. 
The error in the reward function propagates to the value function, scaled by a horizon multiplier $\frac{1}{1-\gamma}$. 
For any joint policy $\pi$, we have: 
\begin{equation}
\left| J^{\pi}_{\hat{r}} - J^{\pi}_{\tilde{r}} \right| \leq \sum_{t=0}^\infty \gamma^t \max_{s, A} \left| \hat{r}(s, A) - \tilde{r}(s, A) \right| \leq \frac{n \epsilon_r(N)}{1 - \gamma}.
\end{equation}

Since $\pi^*$ is an exact Nash equilibrium under the learned reward $\hat{r}$, which is proved in Theorem \ref{thm:2_sum_nash}, by definition, no agent $i$ can improve its return under $\hat{r}$ by changing its own policy to any $\pi'^i$:
\begin{equation}
J^{(\pi'^i, \pi^{-i*})}_{\hat{r}} - J^{\pi^*}_{\hat{r}} \leq 0.
\end{equation}
We now evaluate the potential gain for agent $i$, when changing its policy to $\pi'^i$ under the true target reward $\tilde{r}$. 
By applying the triangle inequality and the established value function error bound, we obtain:
{ \footnotesize
\begin{equation}
\begin{aligned}
J^{(\pi'^i, \pi^{-i*})}_{\tilde{r}} - J^{\pi^*}_{\tilde{r}} 
& = J^{(\pi'^i, \pi^{-i*})}_{\tilde{r}} - J^{(\pi'^i, \pi^{-i*})}_{\hat{r}} 
+ J^{(\pi'^i, \pi^{-i*})}_{\hat{r}} - J^{\pi^*}_{\hat{r}}
\\
& \quad
+ J^{\pi^*}_{\hat{r}} - J^{\pi^*}_{\tilde{r}} \\
& \leq \left| J^{(\pi'^i, \pi^{-i*})}_{\tilde{r}} - J^{(\pi'^i, \pi^{-i*})}_{\hat{r}} \right| + 0 + \left| J^{\pi^*}_{\hat{r}} - J^{\pi^*}_{\tilde{r}} \right| \\
& \leq \frac{n \epsilon_r(N)}{1 - \gamma} + \frac{n \epsilon_r(N)}{1 - \gamma} = \frac{2n \epsilon_r(N)}{1 - \gamma}.
\end{aligned}
\end{equation}
}
Let $\epsilon = \frac{2n \epsilon_r(N)}{1 - \gamma}$. 
We have shown that for any agent $i$ and any alternative policy $\pi'^i$, the improvement in return is bounded by $\epsilon$. 
Therefore, $\pi^*$ is an $\epsilon$-Nash equilibrium under the true preference objective, which completes the proof. 
\end{proof}

\revise{Theorem \ref{thm:4_epsilon_nash} establishes that the policy learned from aggregated rewards approximates the true Nash equilibrium. }
However, the effectiveness of the additive aggregation also relies on Assumption \ref{asp:6_bounded_cross_impact}, which ensures that local improvements do not negatively impact the global objective. 
To support this assumption, we introduce a regularization term, which encourages consistency between the local preference and the aggregated global preference derived from the sum of local rewards. 
Formally, we define an auxiliary loss: 
\begin{equation}
\label{eq:reg_loss}
\begin{aligned}
\mathcal{L}_{\text{reg}}^i = 
& \mathbb{E}_{(\sigma^{(0)},\sigma^{(1)},p^i)} 
\bigg[ \max \bigg( 0, 
\\
& \quad\quad \mathbb{I}[\sigma^{(1)} \succ^i \sigma^{(0)}] \cdot \Big( \hat{R} (\sigma^{(0)}) - \hat{R}(\sigma^{(1)}) \Big) \bigg) 
\bigg],
\end{aligned}
\end{equation}
where $\hat{R}(\sigma) = \sum_{i} \sum_{t} \gamma^t \hat{r}_\psi^i(s_t, A_t)$ is the estimated discounted return of the segment for all agents, and $\mathbb{I}[\cdot]$ is an indicator function. 
\revise{This regularization penalizes the case where an agent's locally preferred segment leads to a decrease in the estimated return of all agents, thereby preserving robustness against occasional violations of Assumption \ref{asp:6_bounded_cross_impact} caused by complex scenarios or severe partial observability. }
The total loss for agent $i$'s reward model is then:
\begin{equation}
\label{eq:reward_loss}
\mathcal{L}^i = \mathcal{L}_{\text{pref}}^i + \lambda \mathcal{L}_{\text{reg}}^i,
\end{equation}
where $\lambda$ is a hyperparameter balancing the two terms.

\begin{algorithm}[t]
\caption{\method Algorithm using the QMIX method}
\label{alg:2_practical}
\begin{algorithmic}[1]
\REQUIRE Preference feedback frequency $K$, number of segment pairs per agent per feedback session $M$
\ENSURE Agent-specific reward models $\hat{r}_\psi^i \; (i=1,\cdots, n)$, local Q-functions $Q_\theta^i  \; (i=1,\cdots, n)$, monotonic mixing network $g_\phi$
\STATE Initialize parameters $\psi, \theta, \phi$, replay buffer $\mathcal{D}$, preference buffers $\mathcal{D}_p^i$ for $i=1,\dots,n$
\STATE Initialize feedback counter $C \gets 0$
\FOR{each iteration}
    \FOR{each environment step $t$}
        \STATE Observe global state $s_t$
        \STATE Each agent selects action $a_t^i$ based on $Q_\theta^i(o_t^i, \cdot)$ using $\epsilon$-greedy exploration
        \STATE Execute joint action $A_t = (a_t^1, \dots, a_t^n)$, observe next state $s_{t+1}$
        \STATE Compute global reward {\small $\hat{r}(s_t, A_t) = \sum_{i=1}^n \hat{r}_\psi^i(s_t, A_t)$}
        \STATE Store transition $(s_t, A_t, s_{t+1}, \hat{r}(s_t, A_t))$ in $\mathcal{D}$
    \ENDFOR
    \IF{iteration mod $K = 0$}
        \FOR{each agent $i$}
            \STATE Sample $M$ segment pairs $(\sigma^{(0)}, \sigma^{(1)})$ from $\mathcal{D}$
            \STATE Query teacher $i$ for preferences $p^i$ on these pairs
            \STATE Store $(\sigma^{(0)}, \sigma^{(1)}, p^i)$ in $\mathcal{D}_p^i$
            \STATE Update reward model $\hat{r}_\psi^i$ by minimizing $\mathcal{L}^i$ (Eq. \ref{eq:reward_loss}) using $\mathcal{D}_p^i$
        \ENDFOR
        \STATE Update the global reward function $\hat{r}(s, A) = \sum_i \hat{r}_\psi^i(s, A)$
        \STATE Relabel all rewards in $\mathcal{D}$ using the updated $\hat{r}$
        \STATE Update $C \gets C + M \cdot n$
    \ENDIF
    \FOR{each gradient step}
        \STATE Sample a minibatch of transitions from $\mathcal{D}$
        \STATE Update local Q-functions $Q_\theta^i$ and mixing network $g_\phi$ using QMIX's loss function
    \ENDFOR
\ENDFOR
\end{algorithmic}
\end{algorithm}

\subsection{MARL with Learned Reward Models}
\label{subsec:marl_with_learned_rewards}

Having constructed the agent-specific reward models $\hat{r}_\psi^i$, we now integrate them into a cooperative MARL algorithm. 
We aggregate the local rewards into a global reward signal via summation
\footnote{
\revise{While the additive form resembles Value Decomposition Networks (VDN) \cite{vdn}, the motivation and scope are distinct. VDN decomposes a known global Q-function for credit assignment under a given global reward, while \method aggregates learned local reward models to construct a global reward signal from decentralized preferences. }
}
as justified by Lemma \ref{lem:monotonicity} and Theorem \ref{thm:2_sum_nash}, \ref{thm:4_epsilon_nash}:
\begin{equation}
\hat{r}(s, A) = \sum_{i=1}^n \hat{r}_\psi^i(s, A).
\end{equation}
This additive formulation satisfies the monotonicity requirement $\partial \hat{r}_{\psi} / \partial \hat{r}^i_{\psi} > 0$ and forms a coherent reward signal. 
We use this reward estimation to replace the environment's predefined global reward.
This allows any MARL algorithm that operates on centralized rewards during training to be applied. 
In this work, we adopt QMIX \cite{qmix} for policy optimization, while other CTDE methods, such as VDN \cite{vdn} or MAPPO \cite{mappo} can also be used.

Specifically, we employ two practical techniques to enhance sample efficiency and performance:

\begin{itemize}
    \item \textbf{Continuous preference collection:} 
    Instead of collecting a fixed batch of preference data before training, we periodically collect preferences using the newly generated trajectories during policy learning \cite{pebble}. 
    This enriches the preference dataset with trajectories generated by the current policy, and improves the reward model's accuracy in regions of the state-action space relevant to the latest policy.

    \item \textbf{Experience relabeling:} 
    After each update of the reward models, we relabel the rewards in the replay buffer using the updated global reward estimation. 
    This allows historical transitions to be reused with the most recent reward estimates, increasing the sample efficiency of both the preference data and the environment interactions \cite{pebble}.
\end{itemize}

Algorithm \ref{alg:2_practical} outlines the complete \method procedure, which uses QMIX \cite{qmix} for policy optimization. 
Specifically, in lines 4-10, agents interact with the environment using the current policy, and transitions are stored with rewards computed by the current global reward model. 
Lines 12-22 indicate the periodic preference collection and reward model updates. 
Each agent's reward model is trained on its respective preference buffer. 
After updating the reward models, the global reward is recomputed, and the replay buffer is relabeled. 
Finally, lines 24-27 perform policy improvement using QMIX, which learns a centralized but factorized Q-function that respects the monotonic constraint. 
This iterative process allows \method to simultaneously refine reward models and policy, 
As more preference feedback is collected, coordination progressively improves.

\begin{table}[t]
\caption{Hyperparameter settings for \method.}
\label{tab:hyperparam}
\centering
\begin{tabular}{lc}
\toprule
\multicolumn{1}{c}{\textbf{Hyperparameter}} & \textbf{Value} \\
\midrule
Preference collection frequency $K$ & 500 \\
Segment pairs per collection $M$ & 300 \\
Reward‑model hidden layers & 2 \\
Hidden units per layer & 128 \\
Reward‑model learning rate & $3\times 10^{-4}$ \\
Regularization coefficient $\lambda$ & $2\times 10^{-2}$ \\
Discount factor $\gamma$ & 0.99 \\
Batch size & 256 \\
Policy learning rate & $7\times 10^{-4}$ \\
Q‑target soft‑update rate $\tau$ & $1\times 10^{-4}$ \\
Mixer hidden dimension & 32 \\
Optimizer & Adam \\
Total environment steps & $1\times 10^{6}$ / $3\times 10^{6}$ \\
Segment length $H$ & 3 / 5 / 10 / 20 \\
Replay buffer size & 5000 \\
$\epsilon$‑greedy start/finish & 1.0 / 0.05 \\
\bottomrule
\end{tabular}
\end{table}

\section{Experimental Results} \label{sec:experiment}

\subsection{Setups}
\label{subsec:setup}

In this section, we evaluate the effectiveness of \method through experiments. 
We test \method on benchmark multi-agent tasks \cite{mordatch2017emergence,lowe2017multi} to compare it fairly with other methods.
We also tried it on a custom mobile phone production task, which shows its potential for real-world industrial applications.

\textbf{Baselines. }
We compare \method with QMIX \cite{qmix} with predefined task rewards. 
This baseline serves as an upper-bound oracle, since it optimizes the true reward function provided by environment developers, representing the best possible performance achievable through explicit reward engineering. 
As a widely used multi-agent MARL baseline, QMIX employs a monotonic value decomposition architecture that is compatible with our reward aggregation mechanism, which ensures a fair comparison between preference-based and reward-based optimization. 
This design validates our central claim that preferences can effectively guide multi-agent policy learning compared to engineered rewards.

\textbf{Experimental details.} 
We implement \method using the CTDE framework, which is introduced in \ref{subsec:pre_marl}. For policy optimization, we base it on QMIX \cite{qmix}.
To generate preferences (Line 14 of Algorithm \ref{alg:2_practical}), we simulate teachers using simple role-specific rules. Each agent’s preferences come from comparing trajectory segments. The comparison follows hand-made metrics that consider the agent’s role, individual effort, and teamwork.
The details of these rules are explained in the following subsections for each task.
All hyperparameters for \method are listed in Table \ref{tab:hyperparam}, and the source code is released\footnote{Code: \url{https://github.com/MoonOutCloudBack/MAGPIE_PbRL}}.

\begin{figure*}[t]
    \centering
    \includegraphics[width=1.0\textwidth]{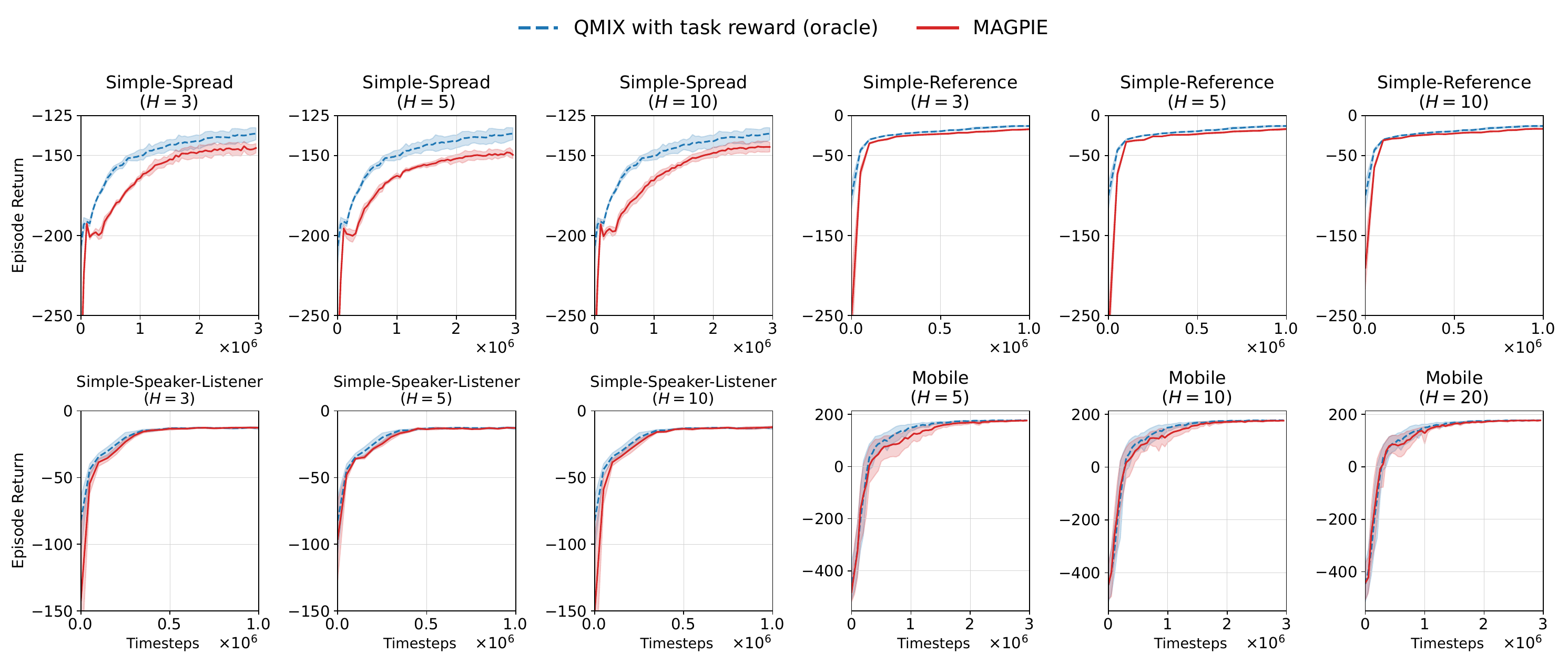}
    \caption{
    The training curves for the episode return of the two methods. 
    Blue: the oracle method (QMIX). Red: our method (\method). 
    Results are averaged over 5 random seeds, with shaded areas indicating standard error. 
    \method achieves comparable performance with QMIX using the true task reward in various tasks. 
    This result validates \method's effectiveness in utilizing decentralized preferences to guide policy learning in multi-agent scenarios.
    }
    \label{fig:performance} 
\end{figure*}

\subsection{Experimental Results on Multi-Agent Benchmark Tasks}
\label{subsec:results_mpe}

\textbf{Task Description. }  
We evaluate \method on three benchmark tasks from the Multi-Agent Particle Environment (MPE) \cite{mordatch2017emergence,lowe2017multi}: Simple-Spread, Simple-Speaker-Listener, and Simple-Reference. 
These tasks progressively increase in heterogeneity and coordination difficulty, providing a suitable testbed for validating whether a preference‑driven framework can align decentralized feedback with global teamwork. 
\begin{itemize}
\item Simple‑Spread involves three homogeneous agents that must cover three distinct landmarks while avoiding collisions. 
Each agent observes only local velocity and relative position observations, and receives a global reward based on the minimum distance to landmarks, with collision penalties. 
This task reflects the credit‑assignment challenge in cooperative MARL, where a single scalar reward must be distributed among agents with identical roles.
\item Simple‑Speaker‑Listener introduces role heterogeneity: a stationary speaker that broadcasts a discrete landmark ID and a mobile listener that navigates to the correct landmark. 
The two agents have disjoint observation spaces; the speaker observes only the target ID, while the listener observes velocities, landmark positions, and communication signals. 
They execute fundamentally different action types, i.e., communication vs. movement. 
This setting directly tests a method's ability to integrate agent‑specific preferences into a coherent policy when agents are not directly comparable.
\item Simple‑Reference further raises the asymmetric partial observability: two agents know only their own target landmarks and must cooperate to reach both. 
Each agent observes only a subset of landmark information, and the reward combines local proximity measures with a global coordination objective. 
The task examines whether decentralized, role‑aware preferences can guide agents to implicitly coordinate under limited information.
\end{itemize}
In all tasks, we utilize discrete action spaces, including movement directions for mobile agents and communication signals for speakers.

\textbf{Preference generation. }
Following previous works \cite{pebble, mu2025clarify}, we generate preference signals through role-specific heuristic rules aligned with each agent's observation and task responsibilities. 
For Simple-Spread, each agent's preferences are derived from its navigation efficiency to assigned landmarks, measured by the minimized distance to all targets while penalizing collisions with peer agents. 
In Simple-Speaker-Listener, the speaker's preferences prioritize accurate communication consistency between transmitted and target landmark IDs, while the listener's preferences emphasize navigation precision to landmarks identified through received signals. 
For Simple-Reference, agents develop preferences based on their distinct partial observations: each agent evaluates trajectory segments through localized target proximity metrics while implicitly coordinating to satisfy global constraints. 
These role‑aware heuristics produce interpretable, actionable preference labels for each agent. 
Preferences are automatically generated by comparing the cumulative heuristic scores of two trajectory segments.

\textbf{Selection of segment length $H$. }
We evaluate \method with segment lengths $H = 3, 5, 10$ across three MPE tasks to account for different temporal granularities in multi-agent behaviors. 
Specifically, $H = 3$ captures short-term interactions, $H = 5$ reflects medium-term coordination, and $H = 10$ represents long-term strategic planning. 
The results confirm that \method achieves robust learning across diverse segment lengths, enabling it to integrate agent-specific preferences without requiring precise hyperparameter tuning, which is critical for real-world scenarios with unknown optimal horizons.

\textbf{Experimental results.}
As shown in Fig. \ref{fig:performance}, \method demonstrates performance comparable to the baseline. 
Specifically, in Simple‑Spread, \method converges slightly slower and attains a marginally lower final return than the oracle. 
This result is expected, as \method first learns a reward model from preferences before conducting policy learning. 
In Simple‑Reference, \method exhibits a rapid initial learning phase followed by steady improvement similar to the baseline, indicating efficient reward-model acquisition prior to policy refinement. 
In Simple‑Speaker‑Listener, the learning trajectory closely matches the oracle, demonstrating \method's capability to handle heterogeneous agents effectively.
These results validate \method's ability to coordinate multi‑agent behavior using only decentralized, role‑specific preferences without relying on predefined rewards. 
The proposed monotonic aggregation mechanism effectively aligns local preferences with the global objective, while continuous reward-model updates enable efficient policy optimization. 
By leveraging preference-driven coordination, \method offers a promising alternative to traditional reward engineering in complex multi-agent reinforcement learning systems.

\subsection{Experimental Results on a Mobile Phone Production Task}
\label{subsec:results_mobile}

\textbf{Task Description. }
To evaluate \method in a real-world industrial setting, we design a custom multi‑agent environment that simulates a sequential mobile phone production line with three main stages: Surface Mount Technology (SMT) \cite{sawik2002balancing}, Assembly \cite{komaki2019flow}, and Test \cite{wu2008modeling}.
Each stage is managed by one agent. Tasks randomly arrive at the first stage and move step by step through each stage until they are finished.
Each agent handles scheduling for its stage. They choose processing speeds that trade off throughput against energy consumption.
All agents work cooperatively. The total reward for the system depends on production efficiency, buffer management, and energy usage.

\begin{itemize}
\item \textbf{State space}:  
The state includes the maximum buffer capacities and current buffer occupancies of all three stages. 
All agents receive the same state as global observation to enable centralized training, while execution remains decentralized. 
Formally, the state vector is defined as:
\begin{equation*}
s = [B_{\text{max}}^{\text{SMT}}, B_{\text{max}}^{\text{Asm}}, B_{\text{max}}^{\text{Test}}, B_{\text{cur}}^{\text{SMT}}, B_{\text{cur}}^{\text{Asm}}, B_{\text{cur}}^{\text{Test}}],
\end{equation*}
where $B_{\text{max}}^i$ is the maximum capacity of stage $i$, and $B_{\text{cur}}^i$ denotes the current number of tasks queued at that stage.

\item \textbf{Action space}: 
The action of each agent $i$ is discrete, denoted as $a_i \in \{0, 1, 2\}$, representing a processing speed level: low, medium, or high. 
The corresponding processing rate $p_i(a_i)$ and energy cost $e_i(a_i)$ vary by stage, which reflects the physical and operational differences among SMT, Assembly, and Test equipment. 
Higher speed levels increase throughput but lead to higher energy consumption, which creates a non-linear trade-off between efficiency and energy use. 
The specific actions for each stage show the heterogeneity of agents. This also makes it hard to decentralized coordination, as each agent must balance its local productivity with its impact on downstream stages and overall system energy consumption.

\item \textbf{Transition}:  
At each timestep, new tasks arrive at the SMT stage following a Poisson distribution. 
Each agent $i$ can process up to $p_i(a_i)$ tasks from its own buffer. Completed tasks are passed to the next stage, except at the final Test stage, where they are considered finished. 
Buffer overflows are truncated and penalized. 
The buffer dynamics for stage $i$ are:
\begin{equation*}
B_{\text{cur},t+1}^i = \min\Big(B_{\text{max}}^i, \max\big(0, B_{\text{cur},t}^i - p_i(a_i) \big) + \text{inflow}_t^i\Big),
\end{equation*}
where $\text{inflow}_t^i$ is the number of tasks arriving from the previous stage, or, for SMT, from the external arrival process. 
The stochastic arrivals and coupled buffer dynamics create a challenging coordination problem under uncertainty.

\item \textbf{Reward function}:  
The shared reward combines three objectives: encouraging task completion, penalizing buffer overflows, and discouraging high energy consumption. At each timestep,
\begin{equation*}
r(s_t, a_t) = \alpha \cdot n_{\text{complete}} - \beta \cdot \sum_{i=1}^{3} o_t^i - \gamma \cdot \sum_{i=1}^{3} e_i(a_i),
\end{equation*}
where $n_{\text{complete}}$ is the number of tasks finished at the Test stage, $o_t^i = \max(0, B_{\text{cur},t}^i - B_{\text{max}}^i)$ denotes overflow at stage $i$, and hyperparameters are set to $\alpha=0.5$, $\beta=1.0$, and $\gamma=0.02$ after scaling.

\end{itemize}
The task is episodic with a horizon of $250$ timesteps, and the cumulative reward over an episode reflects the overall coordination effectiveness of the agent team.

\textbf{Preference generation. }
We follow the method used in the MPE experiments. Preference signals are created using simple rules for each role.
For the SMT agent (first stage), preferences prioritize balanced throughput, considering both its local processing efficiency and the overall assembly-line completion rate. 
Penalties are applied for buffer overflows at the station or in the Assembly stage, weighted by their respective energy costs. 
For the Assembly agent (middle stage), preferences emphasize contributions to local efficiency and final task completion. 
Penalties are incurred for buffer overflows at the station or in the Test stage, with weights based on corresponding energy costs. 
For the Test agent (final stage), preferences focus on terminal task completion, penalizing violations of local buffer constraints, as this stage directly determines the finished output.  
Preferences are then automatically produced by comparing cumulative heuristic scores of trajectory segment pairs.

\textbf{Selection of segment length $H$. }
Similar to the MPE experiments, we evaluate \method using segment lengths $H = 5, 10, 20$, which represent short-term operational decisions, medium-term workflow coordination, and long-term production planning, respectively. 
This range enables an analysis of \method's ability to extract meaningful local preferences from trajectory segments with different temporal granularities. 
The consistent performance across all $H$ values highlights \method's robustness in incorporating role-specific feedback.

\textbf{Experimental Results. }
Figure \ref{fig:performance} presents the learning curves of \method and the baseline for the mobile phone production task. 
Across all configurations, \method achieves final performances comparable to the oracle. 
The results also reveal the possible impact of segment length $H$ on learning: while all variants converge to similar returns, larger $H$ values slightly accelerate learning, which supports the intuition that longer horizons provide more reliable signals for policy comparison. 
This aligns with Theorem \ref{thm:1_nash_iteration}, which states that sufficiently long segments ensure finite-horizon return differences reflect infinite-horizon improvements. 
In summary, Sections \ref{subsec:results_mpe} and \ref{subsec:results_mobile} demonstrate that \method achieves oracle-level performance, effectively handles complex industrial tasks, and maintains stability across design choices, which shows its potential for real-world deployment in reward-free multi-agent systems.

\subsection{\revise{More Experimental Results}}

\begin{figure*}[t]
    \centering
    \includegraphics[width=1.0\textwidth]{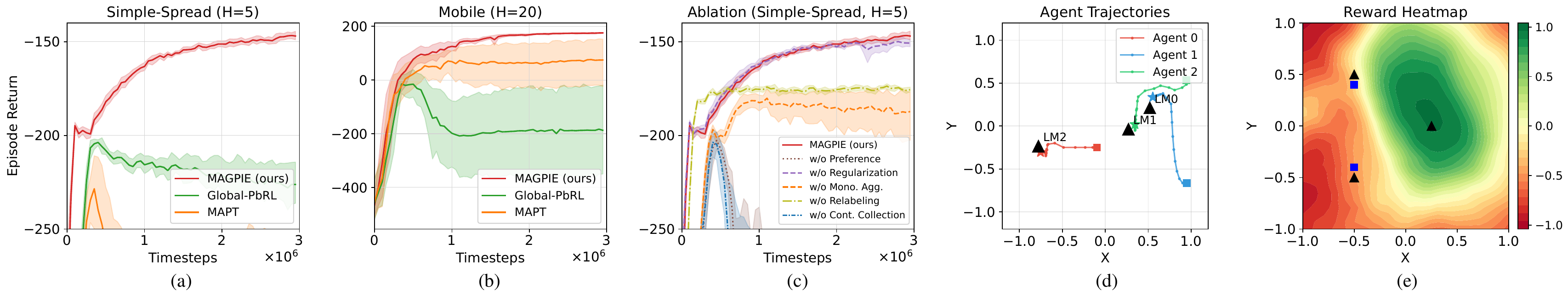}
    \caption{\revise{
    (a, b) Performance comparison of \method and preference-based baselines on the Simple-Spread ($H=5$) and Mobile task ($H=20$).
    (c) Ablation results on the Simple-Spread task ($H$ = 5). 
    (d, e) Visualizations of \method's policy behavior and reward model heatmap on the Simple-Spread task.
    }}
    \label{fig:curve_202606_revise} 
\end{figure*}

\begin{figure*}[t]
    \centering
    \includegraphics[width=1.0\textwidth]{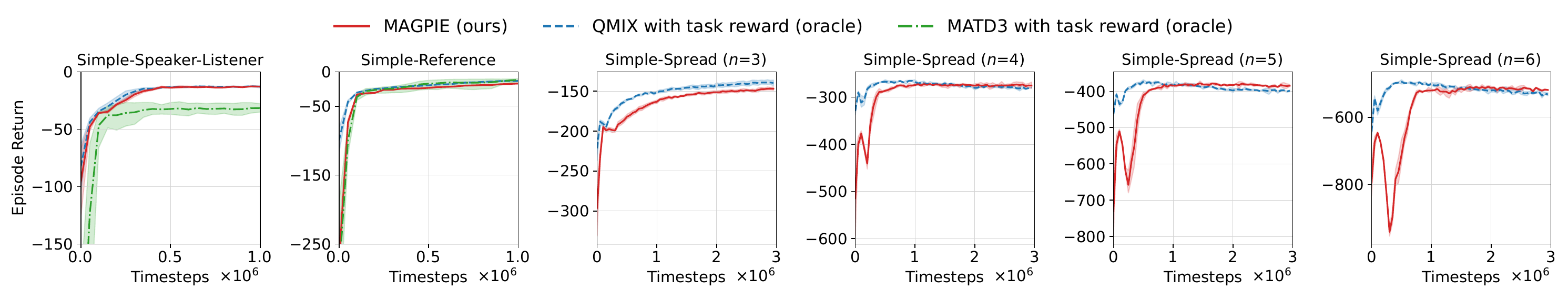}
    \caption{\revise{
    (a, b) Performance comparison of \method and reward-engineered MARL baselines on the Simple-Speaker-Listener task ($H=5$) and the Simple-Reference task ($H=20$). 
    (c, d, e, f) Performance comparison of \method and QMIX on the Simple-Spread task ($H$ = 5) with various numbers of agents. 
    }}
    \label{fig:curve_202606_revise2} 
\end{figure*}

\begin{table}[ht]
\caption{
\revise{Training time (hours) of \method and QMIX.}
}
\label{tab:training_time}
\centering
\begin{tabular}{lrr}
\toprule
Task & \method & QMIX \\
\midrule
Simple-Reference        & 13.6247 $\pm$ 0.0133 & 13.4809 $\pm$ 0.0302 \\
Simple-Speaker-Listener & 9.6424  $\pm$ 0.0166 & 9.5467  $\pm$ 0.0103 \\
Simple-Spread           & 15.4434 $\pm$ 0.0537 & 15.3773 $\pm$ 0.0200 \\
Mobile                  & 37.3630 $\pm$ 0.0019 & 37.3621 $\pm$ 0.0002 \\
\bottomrule
\end{tabular}
\end{table}

\revise{
\textbf{Comparison to more preference-based methods. }
To better show the advantages of \method, we add comparisons with MAPT \cite{zhu2024decoding} and a standard PbRL baseline that learns a single global reward model (Global-PbRL) \cite{bui2025omapl}. 
All methods use the same rule-based preference generation procedure described in Sections \ref{subsec:results_mpe} and \ref{subsec:results_mobile}. 
To simulate human cognitive overload \cite{mu2024sepoa, berner2019dota}, we assign a random global preference label when individual agent preferences within a team conflict. 
Figure \ref{fig:curve_202606_revise}(a, b) shows that \method consistently outperforms both baselines across tasks. 
These results demonstrate the effectiveness of the agent-specific preference decomposition and monotonic aggregation mechanisms in \method. 
}

\revise{
\textbf{Comparison to more reward-engineered methods. }
We compare \method with MATD3 \cite{matd3} using ground-truth reward signals.
Figure \ref{fig:curve_202606_revise2}(a, b) shows that \method achieves comparable performance to MATD3, although MATD3 obtains lower final results in some tasks. 
}

\revise{
\textbf{Evaluation with more agents. }
To demonstrate the feasibility of applying \method in large-scale scenarios, we conducted experiments on the Simple-Spread task with 3$\sim$6 agents. 
Figure \ref{fig:curve_202606_revise2}(c$\sim$f) shows that \method consistently achieves performance comparable to the oracle QMIX baseline, which provides evidence that \method scales to larger-scale scenarios. 
}

\revise{
\textbf{Complexity analysis. }
Compared to standard QMIX, the extra training complexity of \method is $O(n)$.
Specifically, during reward model training, \method trains one reward model $\hat{r}^i_\psi$ per agent independently, which has computational complexity $O(n \cdot N \cdot H)$, with $n$ being the number of agents, $N$ the preference data size per agent, and $H$ the segment length. 
During policy training, \method uses standard QMIX and only replaces the environment reward with the sum of learned reward models.
The extra overhead per transition only involves $n$ forward passes of the reward models and one summation operation.
Let $C_{\text{model}}$ be the cost of a single reward model forward pass; then, the extra complexity of reward computation is $O(n \cdot C_{\text{model}})$, and the summation is $O(n)$.
Additionally, we compared the practical training time of \method against standard QMIX in Table~\ref{tab:training_time}.
The training time of \method is only marginally higher than QMIX, confirming that \method avoids significant runtime overhead.
}

\revise{
\textbf{Ablation studies. }
We conduct ablation studies on four key components of \method: 
(1) the preference-based reward learning module, 
(2) the regularization term enforcing Assumption \ref{asp:6_bounded_cross_impact}, 
(3) the additive monotonic aggregation mechanism (we replace it with multiplicative aggregation), and 
(4) the continuous preference collection and experience relabeling techniques. 
Figure \ref{fig:curve_202606_revise}(c) shows that the preference module, monotonic aggregation, and both training techniques are essential for successful learning. 
Also, removing the regularization term maintains comparable overall performance but causes slight instability during the late stages of training.
This occurs because our expert design explicitly evaluates each agent's contribution to the team objective, making violations of Assumption \ref{asp:6_bounded_cross_impact} naturally rare in fully cooperative tasks. 
Nevertheless, the regularization term provides an additional safety guarantee against potential negative externalities. 
Therefore, we recommend keeping this regularization term to ensure robustness and training stability in more general scenarios. 
}

\revise{
\textbf{Visualizations. }
Figure \ref{fig:curve_202606_revise}(d) visualizes the learned policies of \method in the Simple-Spread environment, which shows three agents successfully covering distinct landmarks without collisions. 
Furthermore, Figure \ref{fig:curve_202606_revise}(e) displays the learned reward heatmap, which shows that the reward of each agent peaks near unoccupied landmarks when the positions of other agents are fixed.
This intuitive pattern confirms the effectiveness of the learned reward model.
}

\section{Conclusion} \label{sec:conclusion}

This paper proposes \method, a multi-agent reinforcement learning method that replaces traditional reward engineering with agent-specific preferences. 
Instead of using complex global reward functions, \method relies on local preference feedback, which provides an intuitive way to train diverse agents in collaborative settings. 
Our theoretical analysis shows that optimizing decentralized preferences can achieve a Nash equilibrium, and the proposed monotonic aggregation mechanism ensures that local improvements contribute to global performance. 
Experimental results across various benchmark tasks and a sequential production line task show that \method is competitive with reward-engineering baselines, which proves its effectiveness and potential for practical applications. 
This paper offers a practical solution for solving multi-agent cooperation tasks in complex scenarios, where defining precise reward signals is challenging.

\section*{Acknowledgment}
This work is supported by the Beijing Natural Science Foundation (L233005), 
the National Key Research and Development Program of China (2022YFA1004600), 
NSFC (No. 62125304, 62192751), 
the Fundamental and Interdisciplinary Disciplines Breakthrough Plan of the Ministry of Education of China (JYB2025XDXM312), 
the 111 International Collaboration Project (B25027), 
and the BNRist project (BNR2024TD03003). 
The author would like to express gratitude to Xun Wang and Haixuan Tong for the illuminating discussions during the preliminary stages of this work.

{
\appendices
\section{Further Analysis of Algorithm Convergence}
\label{appendix:convergence}

In the main text, Theorem \ref{thm:4_epsilon_nash} guarantees that a joint policy maximizing the aggregated reward $\hat{r}$ results in an $\epsilon$-Nash equilibrium for the original problem. 
This result assumes that we can solve the optimization problem perfectly, while in real-world scenarios, the policy optimization is supported by randomly sampled training data, which often leads to errors. 
In this appendix, we further investigate the theoretical properties of the proposed algorithm. 
We prove that the learning process converges under non-idealized, stochastic conditions. 

To make the theoretical analysis feasible, we examine a linear implementation of the monotonic mixing network in QMIX \cite{qmix, vdn}. 
Specifically, we analyze the case where the global Q-value can be decomposed into the sum of local Q-values, which are calculated on agents' local observations and actions, i.e., $Q_{\text{tot}}(s, A) = \sum_{i=1}^n Q_i(o_i, a_i)$. 
In the next subsection, we model the policy update process as a stochastic approximation algorithm and analyze its convergence using the Ordinary Differential Equation (ODE) method proposed in \cite{borkar2000ode}.

\subsection{Problem Setup and Definitions}

Let $\mathcal{S}$ denote the global state space and $\mathcal{A} = \prod_{i=1}^n \mathcal{A}_i$ denote the joint action space. We define $\mathcal{Q} = \{ Q: \mathcal{S} \times \mathcal{A} \to \mathbb{R} \}$ as the space of all possible global Q-functions. 
The infinity norm of this space is denoted by $\|Q\|_\infty = \max_{s, A} |Q(s, A)|$. 

We define the decomposable subspace of the Q-function space $\mathcal{Q}_d \subset \mathcal{Q}$ as follows:
\begin{equation}
    \mathcal{Q}_d = \left\{ Q \in \mathcal{Q} \mid \exists \{Q_i\}_{i=1}^n, Q(s, A) = \sum_{i=1}^n Q_i(s^i, a^i) \right\},
\end{equation}
where $Q_i$ represents the local utility function of agent $i$.
Then, we define the optimal Bellman operator $T: \mathcal{Q} \to \mathcal{Q}$ as: 
\begin{equation}
    (TQ)(s, A) = \mathbb{E}\left[ \hat{r}(s, A) + \gamma \max_{A'} Q(s', A') \mid s, A \right].
\end{equation}
As proved in \cite{RL_sutton}, the optimal Bellman operator $T$ is a $\gamma$-contraction mapping under the infinity norm, i.e., $\|TQ - TQ'\|_\infty \le \gamma \|Q - Q'\|_\infty$.

To update the Q function, the target global Q-value $TQ$ should be projected back into the decomposable subspace $\mathcal{Q}_d$. 
We define $\Pi: \mathcal{Q} \to \mathcal{Q}_d$ as the projection operator. 
Then, the iterative update of the algorithm can be written in the standard stochastic approximation form:
\begin{equation}
\label{eq:sa_update}
Q_{k+1} = Q_k + \alpha_k \left[ \Pi (T Q_k + M_{k+1}) - Q_k \right],
\end{equation}
where $\alpha_k$ is the step size sequence and $M_{k+1}$ is the noise term arising from sampling. 
Since the projection operator $\Pi$ is linear and $Q_k$ is already in $\mathcal{Q}_d$, we have $\Pi Q_k = Q_k$. 
Thus, we can rewrite the update as:
\begin{equation}
    Q_{k+1} = Q_k + \alpha_k \left[ (\Pi T Q_k - Q_k) + \Pi M_{k+1} \right],
\end{equation}
where $\alpha_k$ is the step size of the $k$-th iteration. 
This update can be associated with the following ODE: 
\begin{equation}
    \label{eq:ode}
    \dot{Q}(t) = h(Q(t)) = \Pi T Q(t) - Q(t).
\end{equation}

To establish the theoretical results of convergence, we introduce the following assumptions regarding the environment and the projection operator:

\begin{enumerate}
\item The state space and action space are discrete and finite.

\item The global optimal Q-function $Q^*$ under the aggregated reward $\hat{r}$ lies within the decomposable subspace $\mathcal{Q}_d$. This implies that there exist local functions $Q_i^*$ such that their sum exactly represents the global optimal value. 

\item The projection operator $\Pi$ is non-expansive under the infinity norm. That is, for any $Q, Q' \in \mathcal{Q}$, the inequality $\|\Pi Q - \Pi Q'\|_\infty \le \|Q - Q'\|_\infty$ holds. 

\item The step size sequence $\{\alpha_k\}$ satisfies the Robbins-Monro conditions \cite{robbins1951stochastic}: $\sum \alpha_k = \infty$ and $\sum \alpha_k^2 < \infty$. 
\end{enumerate}

\subsection{Convergence Result}

Based on these assumptions, we provide the following theorem to establish a convergence analysis. 

\begin{theorem}
\label{thm:5_ode}
$Q^*$ is the unique globally asymptotically stable equilibrium point of the associated ODE in \eqref{eq:ode}.
Furthermore, the sequence $\{Q_k\}$ generated by the algorithm converges almost surely to the global optimal Q-function $Q^*$. 
\end{theorem}

\begin{proof}
We use the ODE method proposed by \cite{borkar2000ode}. 
First, we define the composite operator $G = \Pi T$, and analyze the ODE $\dot{Q} = G(Q) - Q$. 
We show that $G$ is a contraction mapping under the infinity norm. 
Specifically, for any $Q_1, Q_2 \in \mathcal{Q}$, we have:
\begin{equation}
\begin{aligned}
    \| G(Q_1) - G(Q_2) \|_\infty &= \| \Pi T Q_1 - \Pi T Q_2 \|_\infty \\
    &\le \| T Q_1 - T Q_2 \|_\infty \\
    &\le \gamma \| Q_1 - Q_2 \|_\infty.
\end{aligned}
\end{equation}
The first inequality follows from the non-expansiveness of $\Pi$, and the second inequality follows from the contraction property of $T$. Thus, $G$ is a contraction mapping with $\gamma < 1$.

We assume in the previous subsection that the optimal value function $Q^*$ belongs to $\mathcal{Q}_d$ and satisfies $T Q^* = Q^*$. 
Since $Q^* \in \mathcal{Q}_d$, we have $\Pi Q^* = Q^*$. Therefore, 
\begin{equation}
    G(Q^*) = \Pi T Q^* = \Pi Q^* = Q^*.
\end{equation}
This implies that $Q^*$ is the unique fixed point of the map $G$. 
Consequently, $Q^*$ is the unique equilibrium point of the ODE in \eqref{eq:ode}, as $h(Q^*) = 0$. 

We construct a Lyapunov function $V(Q) = \|Q - Q^*\|_\infty$, and analyze the derivative of $V(Q(t))$ along the trajectory of the ODE. 
For a small time step $\eta > 0$, we can approximate the trajectory as: 
\begin{equation}
\begin{aligned}
    Q(t+\eta) &\approx Q(t) + \eta (\Pi T Q(t) - Q(t)) \\
    &= (1-\eta) Q(t) + \eta G(Q(t)).
\end{aligned}
\end{equation}
We subtract $Q^*$ from both sides. 
Note that $Q^* = (1-\eta)Q^* + \eta Q^*$ because $G(Q^*) = Q^*$, we get:
\begin{equation}
\begin{aligned}
    & \|Q(t+\eta) - Q^*\|_\infty \\
    \le ~~ & (1-\eta)\|Q(t) - Q^*\|_\infty + \eta \|G(Q(t)) - Q^*\|_\infty \\
    \le ~~ & (1-\eta)\|Q(t) - Q^*\|_\infty + \eta \gamma \|Q(t) - Q^*\|_\infty \\
    = ~~ & (1 - (1-\gamma)\eta) \|Q(t) - Q^*\|_\infty.
\end{aligned}
\end{equation}
Dividing by $\eta$ and letting $\eta \to 0^+$, we obtain the upper right derivative:
\begin{equation}
    \frac{d^+}{dt} V(Q(t)) \le -(1-\gamma) V(Q(t)).
\end{equation}
This inequality shows that the error $V(Q(t))$ decays to zero at an exponential rate of $1-\gamma$. 
Therefore, $Q^*$ is the globally asymptotically stable equilibrium of the ODE. 

Next, we verify the noise condition with random sampling. 
The term $M_{k+1}$ represents the standard Q-learning noise, which is a martingale difference sequence with respect to the history \cite{borkar2000ode}. 
Since $\Pi$ is a linear operator, the projected noise $\Pi M_{k+1}$ remains a martingale difference sequence, i.e., $\mathbb{E}[\Pi M_{k+1} \mid \mathcal{F}_k] = 0$.
Furthermore, since the rewards are bounded and the state-action space is finite, the variance of the noise is bounded by a quadratic function of the current Q-values. 
Specifically, there exists a constant $C_0$ such that $\mathbb{E}[\|M_{k+1}\|^2_\infty \mid \mathcal{F}_k] \leq C_0(1 + \|Q_k\|^2_\infty)$. This satisfies the standard assumption for stability. 

Therefore, the ODE in \eqref{eq:ode} has a unique globally asymptotically stable equilibrium $Q^*$, and the noise satisfies the martingale difference condition. 
According to Theorem 2.2 in \cite{borkar2000ode}, the sequence $\{Q_k\}$ converges almost surely to $Q^*$. 
This completes the proof. 
\end{proof}

\begin{remark}
Theorem \ref{thm:5_ode} establishes that the algorithm converges to the global optimal policy $\pi^*$ for the learned reward model $\hat{r}$. 
We combine this with Theorem \ref{thm:4_epsilon_nash}, which states that the optimal policy for $\hat{r}$ corresponds to an $\epsilon$-Nash equilibrium of the true underlying preferences. 
Therefore, we conclude that under the above assumptions, the proposed method converges to an $\epsilon$-Nash equilibrium policy using decentralized preference feedback.
This provides a theoretical foundation for preference-based multi-agent learning in stochastic environments. 
\end{remark}

}

\bibliographystyle{IEEEtran}
\bibliography{ref/refs_our_works, ref/refs_MARL, ref/refs_PbRL, ref/refs_rl, ref/refs_theory, ref/refs_misc}

\end{document}